\documentclass[11pt,letterpaper]{article}

\usepackage[utf8]{inputenc}
\usepackage{amsmath}
\usepackage{amsfonts}
\usepackage{amssymb}
\usepackage{graphicx}
\usepackage{float}
\usepackage{xcolor}
\usepackage{amsthm}
\usepackage{mathtools}
\usepackage{hyperref}
\usepackage{subcaption}
\usepackage{mwe}
\usepackage{booktabs}
\usepackage{lscape} 
\usepackage{multirow}
\usepackage{comment}

\newcommand{\E}{\mathsf{E}}
\newcommand{\var}{\mathsf{var}}
\newcommand{\cov}{\mathsf{cov}}

\newcommand{\R}{\mathbb{R}}
\newcommand{\1}{\mathbf{1}}

\newcommand{\PR}{\mathsf{P}}

\newcommand\fracp[2]{\frac{\partial {#1}}{\partial {#2}}}

\newtheorem{corollary}{Corollary}[section]

\newtheorem{definition}{Definition}[section]

\newtheorem{proposition}{Proposition}[section]

\newtheorem{remark}{Remark}[section]

\numberwithin{table}{section}
\numberwithin{figure}{section}

\title{\LARGE\bf Should Bank Stress Tests Be Fair?}
\author{Paul Glasserman and Mike Li \\ Columbia Business School} 
\date{May 2023}

\begin{document} 

\maketitle

\begin{abstract}
Regulatory stress tests have become one of the main tools for setting
capital requirements at the largest U.S. banks. The Federal Reserve
uses confidential models to evaluate bank-specific outcomes for
bank-specific portfolios in shared stress scenarios. As a matter of
policy, the same models are used for all banks, despite considerable
heterogeneity across institutions; individual banks have contended
that some models are not suited to their businesses. Motivated by this
debate, we ask, what is a fair aggregation of individually tailored
models into a common model? We argue that simply pooling data
across banks treats banks equally but is subject to two deficiencies:
it may distort the impact of legitimate portfolio features, and it is
vulnerable to implicit misdirection of legitimate information to infer
bank identity. We compare various notions of regression fairness to
address these deficiencies, considering both forecast accuracy and
equal treatment. In the setting of linear models, we argue for
estimating and then discarding centered bank fixed effects as
preferable to simply ignoring differences across banks. We present
evidence that the overall impact can be material. We also discuss
extensions to nonlinear models.
\end{abstract}

\baselineskip18pt

\section{Introduction}

In the aftermath of the 2008 financial crisis, U.S. banking regulators
adopted stress testing as a primary tool for monitoring the capital
adequacy of the largest banks. For each round of annual stress tests,
the Federal Reserve announces a ``severely adverse stress scenario,''
defined by a hypothetical path of economic variables over the next
several quarters. A typical path includes an increase in unemployment,
a decline in GDP, and projections for the level and volatility of the stock
market, among other variables. The largest banks provide the Fed with
detailed information about their loan portfolios and other assets. The
Fed then applies internally developed models to project revenues and
losses for each bank through the stress scenario. Banks are required to
have sufficient capital to weather the projected losses.

The Fed does not disclose details of the models it uses to project
revenues and losses. The Fed makes clear that to ensure consistent 
treatment for different banks it uses ``industry models,'' as opposed
to models tailored to individual banks. As a matter of policy, the same
models are applied to all banks. Quoting Board of Governors \cite{frb}
(p.3), ``two firms with the same portfolio receive the same results for
that portfolio.'' We will refer to this statement as the Fed's principle of
equal treatment.

Banks have countered that the Fed's models fail to capture bank-specific
features that could lower projected losses. They have made these
arguments in requests for reconsideration of stress test results. Of course,
banks are not objective critics of the Fed's supervision; but significant
heterogeneity among the largest banks is indisputable. The banks subject
to annual stress testing include universal banks, investment banks, large
regional banks, the U.S. subsidiaries of certain foreign banks, and a
variety of more specialized financial firms. It is certainly possible that
bank-specific models would produce more accurate forecasts than a
single industry model, in which case using a single model entails a
trade-off between forecast accuracy and consistency across banks. 

The heterogeneity among large banks motivates the questions we study: 
What is the best way to aggregate bank-specific models into an industry
model? How should the Fed's principle of equal treatment be interpreted
and implemented? Is simply ignoring bank identity in estimating and
applying models the best way to achieve fairness? To what extent is fairness
at odds with accuracy? Although the heterogeneity of large banks is widely
recognized, we know of no prior work that seeks to address this property
within the constraints of the Fed's policy of equal treatment. We will argue
that addressing heterogeneity is preferable to ignoring it.

The question of fairness in algorithms and models has received a great
deal of renewed interest in recent years, in some cases reviving earlier
debates over fairness in testing and related policies that were not explicitly
``algorithmic;'' see, for example, the overviews in Barocas, Hardt, and
Narayanan \cite{fairmlbook} and Hutchinson and Mitchell \cite{hutmit}.
We draw on this literature, but our setting differs in important ways from
most discussions of fairness. 

Algorithmic fairness is usually concerned with ensuring that certain
protected attributes --- race or gender, for example --- do not influence
outcomes such as hiring decisions or loan approvals. Different methods
can be compared based on alternative measures of influence and the degree
to which sensitive attributes are indeed protected. 

The counterpart of a protected attribute in our setting is a bank's identity;
but this attribute is not so much protected (in the sense that race and gender
are) as inadmissible for the Fed's purpose. In stating that ``two firms with
the same portfolio receive the same results for that portfolio,'' the Fed is
stating that bank identity is not a legitimate predictor of losses. Perhaps, then,
fairness is achieved as long as the Fed uses the same model for all banks. In
other words, perhaps ``fairness through unawareness,'' paraphrasing
Dwork et al.\ \cite{dwork}, is sufficient in this setting. Moreover, in questioning
whether the Fed's models apply to them, banks are not claiming discrimination;
on the contrary, they are asking for discrimination --- asking that the Fed 
change its models to recognize ways in which an individual bank differs from
other banks.

To investigate these issues, we focus primarily on a simple setting in which the
``true'' loss rate for each bank is described by a bank-specific regression on
portfolio features and scenario features. The regulator's goal is to aggregate
these bank-specific models into a single model. A natural interpretation of an
``industry'' model in this setting is a pooled regression based on combining
results across banks. The pooled model treats banks equally, but we show that it
has at least two significant deficiencies: when applied to heterogeneous banks,
it can produce poor measures of the marginal impact of individual features, even
resulting in the wrong sign; and it implicitly misdirects legitimate information in
portfolio features to infer (or proxy for) bank identity in forecasting losses. The
second of these deficiencies works against the spirit of equal treatment of banks,
even if bank identity is not explicitly used in the model.

We then investigate the application of ideas from algorithmic fairness in our
setting. The fairness literature has mainly focused on classification problems
(hiring decisions and credit approvals, for example), with regression problems
getting somewhat less attention. Chzhen et al.\ \cite{chzhen} and Le Gouic et
al.\ \cite{legouic} developed a method of particular importance for regression that
Le Gouic et al.\ \cite{legouic} call ``projection to fairness.'' This method produces
optimal forecasts (in the least-squares sense) subject to a fairness constraint
known as {\it demographic parity}. We examine the application of this approach in
our setting and conclude that it goes too far in leveling results across banks.

The pooled method ignores fairness and the projection method goes too far in
imposing fairness, so we seek an intermediate solution. Johnson, Foster, and Stine
\cite{jfs} introduce a variety of methods for introducing fairness considerations in
regression. These include methods they call ``full equality of opportunity'' (FEO)
and ``substantive equality of opportunity'' (SEO). We examine these methods in our
setting and conclude that the FEO method provides an attractive solution. In
particular, we show that it addresses the two deficiencies of the pooled method
highlighted above: it removes the distortion in the pooled coefficients that results
from bank heterogeneity, and it prevents the misdirection of legitimate information
to infer bank identity. Indeed, we show that the only way to achieve lower forecast
errors than the FEO method is through such misdirection, a result that sheds light on
the trade-off between accuracy and fairness.

Moreover, the method is easy to interpret and implement: fit a pooled model with
centered bank fixed effects, and then \textit{discard} the centered fixed effects to
forecast losses. Including the fixed effects prevents misdirection of legitimate
information; discarding them is necessary to treat banks equally; centering ensures
that the overall mean forecast remains unchanged. Although we mainly work with
linear models, we show that these ideas can be extended to nonlinear models as well. 
We also derive an extension of FEO to remove certain interaction effects, as opposed
to just fixed effects.

We then investigate the empirical relevance of these considerations. We regress the
loss rates of loan portfolios (credit cards, first lien mortgages, commercial real estate,
and commercial and industrial loans) on measures of portfolio quality (past-due rates
and allowances for losses) and macroeconomic variables. We document significant
heterogeneity across banks in their estimated coefficients, and we show that the
differences between pooled and FEO estimates can be material. This investigation is
limited by the information banks make public --- the Federal Reserve has access to far
more granular data in forecasting losses. We cannot claim to approximate the Fed's
forecasts; our goal is to provide evidence of the potential importance of heterogeneity.

To help position our work, we briefly discuss some other research on bank stress tests.
Covas, Rump, and Zakrajsek \cite{covas}, Kapinos and Mitnik \cite{kapmit}, and Kupiec
\cite{kupiec} find strong evidence of heterogeneity in banks' responses to
macroeconomic shocks, and Kapinos and Mitnik \cite{kapmit} argue that ignoring
heterogeneity can substantially underestimate projected capital requirements. The
related models of Hirtle et al.\ \cite{class} and Guerrieri and Welch \cite{guewel}
forecast aggregate results and are therefore not concerned with differences among
banks. Heterogeneity in the accuracy of the Fed's models for different banks is
suggested by the comparisons in Agarwal et al.\ \cite{agarwal}, Bassett and Berrospide
\cite{basber}, and Flannery, Hirtle, and Kovner \cite{fhk} between the Fed's results and
results based on the banks' own models.

A separate line of research considers the design of stress scenarios. Several studies
(including Breuer et al.\ \cite{breuer}, Flood and Korenko \cite{flokor},
Glasserman et al.\ \cite{gkk}, Pritsker \cite{pritsker}, and Schuermann \cite{til}) have
advocated the use of multiple scenarios to capture different combinations of risk
factors. Cope et al.\ \cite{cope} and Flood et al.\ \cite{flood} recommend designing
scenarios to reflect bank heterogeneity. Parlatore and Philippon \cite{parphi} propose
a theoretical framework for scenario design as a problem of optimal information
acquisition.

Several studies have investigated the information content of stress test results, either
through market responses (as in Fernandes, Igan, and Pinheiro \cite{fip}, Flannery, Hirtle,
and Kovner \cite{fhk}, Georgescu et al.\ \cite{ggkk}, Glasserman and Tangirala
\cite{glatan}, Guerrieri and Modugno \cite{guemod}, Morgan, Peristiani, and Savino
\cite{mps}, and Sahin, de Haan, and Neretina \cite{shn}) or through subsequent bank
performance (as in Kupiec \cite{kupiec} and Philippon, Pessarossi, and Camara \cite{ppc}).
Flannery \cite{flannery} discusses just how much information the Fed should disclose
about stress testing procedures and outcomes. For perspectives on the effectiveness of the
Fed's stress tests, see Kohn and Liang \cite{kohlia} and Schuermann \cite{til}.

We provide additional background on the Federal Reserve's stress tests in
Section~\ref{s:background}. Section~\ref{s:pool} lays out our modeling framework and
analyzes the pooled industry model within this framework. Section~\ref{s:freg} analyzes
various ways to introduce fairness considerations, including the projection-to-fairness
and FEO methods. Section~\ref{s:nonlinear} considers nonlinear models.
Section~\ref{sec:experiment} presents our empirical results. Proofs of our main results
appear in the appendix. Additional supporting
theoretical (Sections~\ref{s:cross}--\ref{a:cxcomb}) and
empirical (Sections~\ref{appendix:data}--\ref{appendix:ppnr}) material is included in
the Electronic Companion. Most of our discussion considers loss models, but we consider
revenue models in Section~\ref{appendix:ppnr}.

\section{Background}
\label{s:background}

This section provides background on the Federal Reserve's stress testing process and on
the heterogeneity of the participating banks.

\subsection{Regulatory Bank Stress Tests}
\label{s:back1}

In early 2009, in the depths of the Global Financial Crisis, the Federal Reserve launched a
stress test of the 19 largest U.S. bank holding companies to gauge how much more capital
they would need if economic conditions continued to worsen. The results of the stress test
were made public, and the transparency and credibility of the process have been credited
with restoring public confidence and helping to end the crisis.

The Dodd-Frank Act, the package of reforms that followed the crisis, codified the use of
stress testing for bank supervision. The number of banks subject to DFAST (Dodd-Frank
Act Stress Tests) has varied over time. The current requirement applies annually to banks
with over \$250 billion in assets and every other year to banks with assets between \$100
billion and \$250 billion. The 2022 DFAST covered 34 banks. We refer to the participating
firms as ``banks,'' but they are more precisely holding companies, including the U.S.
subsidiaries of some foreign banks.

The inputs to the stress test analysis are the stress scenario, which is
common to all banks, and bank-specific balance sheet information.
A scenario is specified through a hypothetical path of economic variables
over the next 13 quarters. The 2022 DFAST specified paths for 28
variables, including GDP, inflation, unemployment, stock market
and real estate indexes, interest rates, exchange rates, and measures
of overseas economic activity. Each bank submits detailed information
on its loans and other assets.

The Fed uses 21 models to integrate the stress scenarios with bank-level
information to make bank-level projections. For example, one model
applies to credit cards, one to first lien residential mortgages, one to
commercial real estate loans, and another to commercial and industrial loans.
These models project losses in each of these portfolios. Some other models
project revenues.

The Fed does not disclose details of its models, either to banks or the general public.
But it does describe its general modeling approach in public documents. At a high
level, a model assigns a loss rate to a set of bank-specific loan portfolio features
$x$ and a common set of scenario variables $z$ through a function $f(x,z)$. The
function $f$ is estimated from past observations of the macro variables and 
portfolio features for multiple banks. Thus, $f$ is estimated as an industry-wide
model and then applied individually to each bank.

This approach is described, for example, on p.3 of Board of Governors \cite{frb}, where
we read, ``The Federal Reserve generally develops its models under an industry-level
approach calibrated using data from many financial institutions.\dots The Federal
Reserve models the response of specific portfolios and instruments to variations in
macroeconomic and financial scenario variables such that differences across firms
are driven by differences in firm-specific input data, as opposed to differences in
model parameters and specifications. As a result, two firms with the same portfolio
receive the same results for that portfolio in the supervisory stress test, facilitating
the comparability of results.''

As noted in the introduction, we refer to the principle that banks with the same
portfolio receive the same results as {\it equal treatment}.

\subsection{Bank Heterogeneity}
\label{s:back2}

The appropriateness of equal treatment seems incontrovertible. 
But the right notion of consistency across firms becomes
less clear when portfolios vary widely, and the largest U.S. banks
are a highly heterogeneous group. We may not expect a regional
bank to have an investment bank's skill in the capital markets, nor
do we expect the investment bank to have the regional bank's skill
in making single-family residential loans. 

Heterogeneity among large banks is illustrated in Figure~\ref{f:gsib}.
The left panel applies to the banks that participated in the Federal
Reserve's 2022 stress test. It shows the distribution of the banks by their
Global Industry Classification Standard sub-industry classifications.
This group includes diversified banks (such as JPMorgan Chase and
Bank of America); regional banks (like PNC Financial and Citizens Financial);
consumer finance companies (including American Express and Discover);
custody banks (such as Bank of New York Mellon and State Street); 
investment banks (including Goldman Sachs and Morgan Stanley);
and intermediate holding companies comprising the U.S. subsidiaries
of foreign banks (such as TD Group and Credit Suisse USA). The
distribution of banks across categories reflects important differences
in their areas of specialization.

The right panel of Figure~\ref{f:gsib} applies to the U.S. Global
Systemically Important Banks (G-SIBs). It shows heterogeneity in the
fractions of loans the banks hold in each of four categories.
For example, for Wells Fargo (WFC) first lien mortgages are a relatively
large fraction of its loans, whereas for Citigroup (C), credit cards make
up a relatively large fraction. The figure suggests different areas of
specialization in lending, even among the largest U.S. banks.

\begin{figure}
\centerline{\includegraphics[trim={.25in 1.25in .25in 1.25in},clip,width=3.25in]{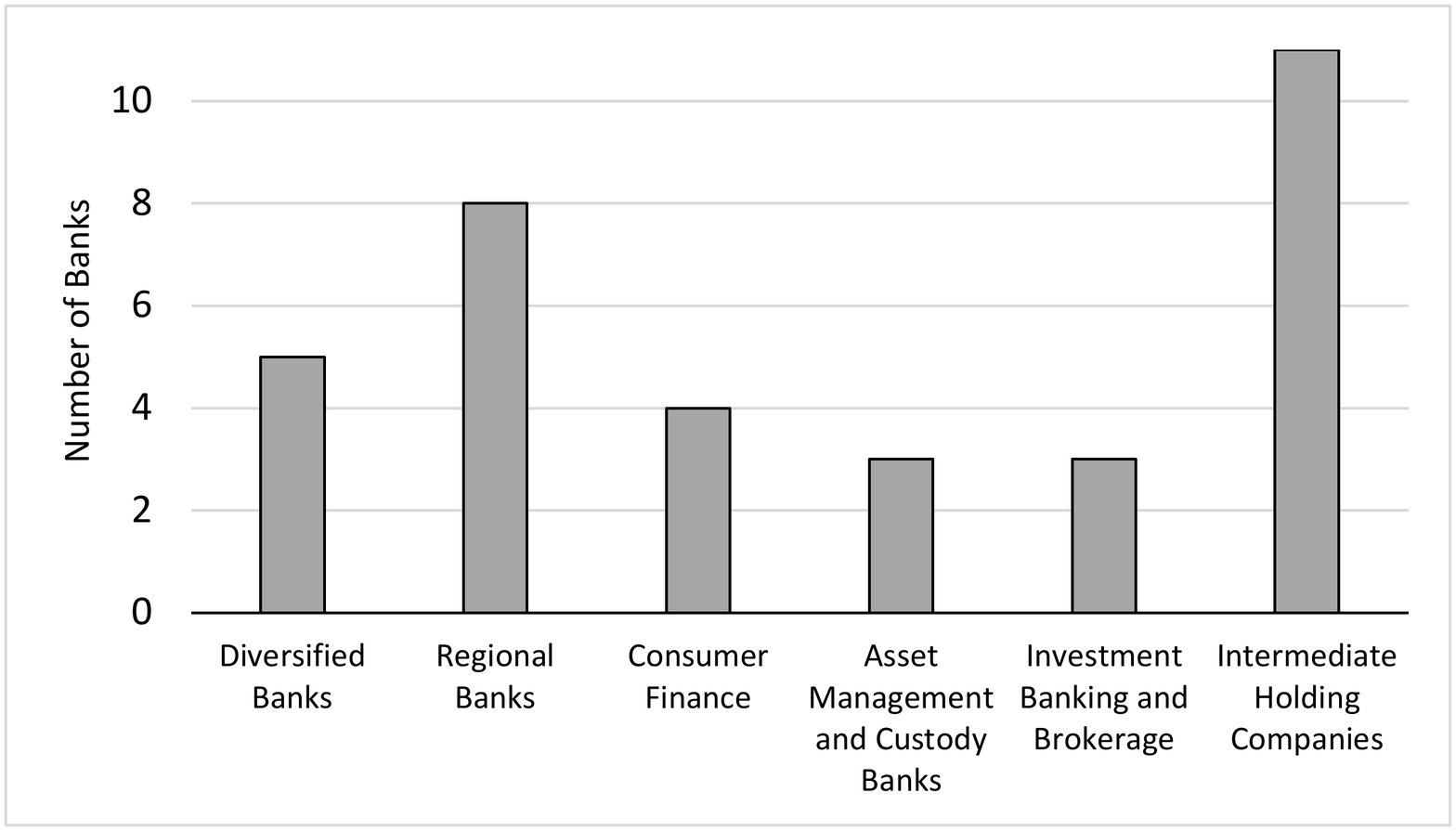}\quad\includegraphics[trim={.25in 1.25in .25in 1.25in},clip,width=3.25in]{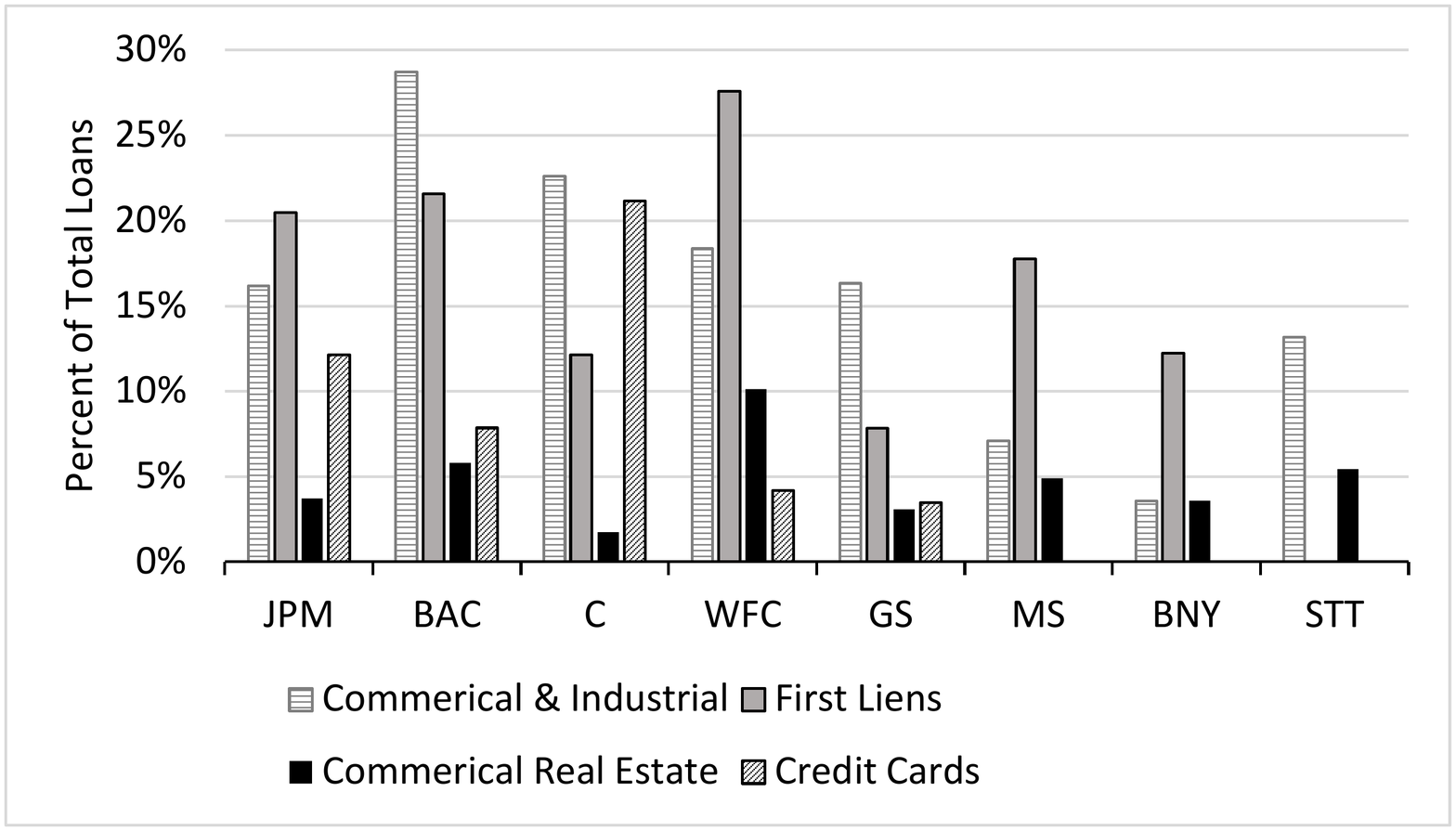}}
\caption{Heterogeneity among large banks. 
Left: Distribution of 2022 stress test banks by GIC sub-industry. 
Right: The percentage of loans in each of four categories for each of the U.S.
G-SIBs, based on Y-9C reports for Q4 2021.}
\label{f:gsib}
\end{figure}

Beginning in 2020, the Federal Reserve allowed banks to submit requests
for reconsideration of the stress capital buffer set by the Fed through
the stress testing process. (The capital buffer is set through the
Comprehensive Capital Analysis and Review, or CCAR, process, which
accompanies the stress test.) The banks' requests are confidential,
but the Fed's responses to these requests are public. The responses show
that the banks were arguing for reconsideration at least in part based on
claims that the Fed's models do not capture distinctive features of the banks'
businesses. For example, Regions Financial claimed that the Fed's models
overlook the bank's hedging of interest rate risk. Goldman Sachs claimed
that the Fed's models omit information relevant to the credit quality of the
bank's mortgage loans. Citizens Financial claimed that the Fed's models 
overlook the bank's loss-sharing agreements in its retail portfolio.

Five firms requested reconsideration in 2020, and all five requests were rejected.
In its response%
\footnote{https://www.federalreserve.gov/supervisionreg/files/goldman-sachs-group-inc-20200904.pdf}
to Goldman Sachs, the Fed wrote,
``the Board has determined that it will follow its published principles for stress testing,
including the principle of creating industry-level models, and not modify the existing
results of these models. In particular, models used in the supervisory stress test are
generally developed according to an industry-level approach, calibrated using data 
from many institutions.'' Similar statements appear in all five rejections. These exchanges
point to a debate in which the banks highlight their heterogeneity and the Fed asserts
the importance of consistency.

\subsection{Heterogeneity and Fairness}
\label{s:hf}

Read narrowly, the principle of equal treatment --- the Fed's statement that 
``two firms with the same portfolio receive the same results for that portfolio''
--- is easy to satisfy. It holds in any model that forecasts losses based only on
portfolio features and the stress scenario, without using any other bank-specific
information. Even this narrow reading has important implications. For example,
the quality of a bank's IT systems or the strength of its ``culture''%
\footnote{For a perspective on the importance of culture, see, for example, ``Enhancing Financial Stability by Improving Culture in the Financial Services Industry,''
a speech given by then president of the Federal Reserve Bank of New York, William C. Dudley,
on October 20, 2014, https://www.newyorkfed.org/newsevents/speeches/2014/dud141020a.html.}
may be important factors in determining a bank's losses under stress, but as
they are not features of a loan portfolio, 
the Fed's modeling principle would preclude incorporating them into the Fed's models.
Matters like the quality of a bank's internal governance and controls must be
addressed in other parts of the overall bank supervision process, outside of
stress testing.%
\footnote{The Fed's stress tests previously included a qualitative component, but this component was dropped in 2019.}
Within the Basel framework, these considerations are part of the Pillar 2
supervisory process, as described in BCBS \cite{bcbs}.

Under a broader interpretation of the principle of equal treatment, a regulatory
model should also exclude indirect proxies for bank identity. Suppose, for
example, that a bank with outdated IT systems had a particularly large number
of loans to the energy sector. Suppose further that because of its weak IT the
bank was a poor monitor of its borrowers and suffered abnormally large losses
in downturns. With information about IT excluded, a predictive model of losses
that uses this history would likely overstate the risk of loans to the energy sector.
This outcome is arguably unfair to all banks making energy loans, in that they
would be indirectly penalized for one bank's weak IT. Addressing these types of
indirect effects drives our investigation. In its narrow sense, equal treatment
requires an indifference to which banks hold which portfolios once a model is
selected; the broader interpretation seeks to remove the influence of bank
identity in the design of the model.

The Fed's stated principle implicitly responds to concerns for \textit{disparate
treatment} of banks. The broader interpretation --- precluding proxies for bank
identity --- aligns with a concern for a particular notion of \textit{disparate impact} 
used in the literature on algorithmic fairness (see, for example, Chapter 6 of
Barocas et al.\ \cite{fairmlbook}, Section 3 of Lipton et al.\ \cite{lipton}, and 
Prince and Schwarcz \cite{prince}), sometimes called ``proxy discrimination.''
The banks' objections, as reflected in their reconsideration requests, can be seen
as concerns for a different type of disparate impact: even if the same model is
applied to all banks, and even if the model is free of bank-identity proxies, some
banks may claim to be more adversely affected than others by the model's
limitations. Most of the objections raised by banks can be understood as 
pointing to omitted variables --- features omitted from the Fed's models that a
bank believes would result in a more favorable outcome if included in the models. 
The Fed's responses suggest a reluctance to incorporate overly narrow features 
into models, particularly features that might affect only a single bank. 
Model limitations, of the type claimed by the banks are likely inevitable, given the
limited data available on bank performance in scenarios of severe stress. The Fed
should strive to continue to improve its models, but our concern is not primarily
for the banks' objections. Our focus is rather on how best to interpret and implement
the Fed's stated principle of equal treatment, particularly under the broader
interpretation that addresses elements of both disparate treatment and disparate
impact, within the overarching goal of accurately forecasting stressed losses for
each bank.

The fairness literature distinguishes notions of individual fairness and group
fairness, where the members of a group often share a sensitive or protected attribute.
More abstractly, an individual is defined by a fixed set of features, and a group is
characterized by a probability distribution over features; see, for example, the
characterizations of individuals and groups in Sections 2 and 3 of Dwork et al.\ 
\cite{dwork}. From this perspective, individual fairness is concerned with fairness
conditional on a set of features, whereas measures of group fairness incorporate
distributions over features. Individual fairness typically requires that individuals with
similar features be treated similarly. For a portfolio loss model, this condition is
satisfied if the predicted loss is a suitably smooth function of the features of
individual portfolios. But group fairness is more relevant to our setting than individual
fairness because we think of each bank, with its particular mix of businesses and areas
of focus, not as one portfolio but as a distribution over portfolios the bank might hold
at different times. We are interested in accuracy and fairness with respect to these
distributions of bank portfolio features. We therefore view each bank as a group of
individual portfolios that share the attribute of bank identity. In contrast, individual
fairness would be relevant to evaluating accuracy and fairness conditional on a specific
portfolio for each bank. We will make this formulation of groups and individuals more
explicit in the next section after introducing our basic model.

\section{Pooling: Fairness Through Unawareness?}
\label{s:pool}

\subsection{Basic Model}
\label{s:model}

To capture bank heterogeneity, we consider a market with multiple banks, indexed by
$s=1,\dots,\bar{S}$. The loss rate (or net charge-off rate) $Y_s$ for bank $s$ is given by
\begin{equation}
Y_s = \alpha_s + \beta_s^{\top}X_s + \epsilon_s,
\label{ys}
\end{equation}
with $\alpha_s\in\R$ and $\beta_s\in\R^d$. Here, $X_s$ is a $d$-dimensional vector
of predictive variables; at this point, we do not distinguish between portfolio
characteristics and macro variables. The portfolio characteristics include information
about a bank's borrowers and loan terms. We use a linear specification in (\ref{ys})
because it offers the simplest setting to explore the interaction of heterogeneity and
fairness; we discuss nonlinear extensions in Section~\ref{s:nonlinear}. We take (\ref{ys})
to be the true relationship between the loss rate $Y_s$ for bank $s$ over the forecast
horizon and characteristics $X_s$ known at the date the forecast is made. Loss rates
are normalized by loan balances to make values of $Y_s$ comparable across banks of
different sizes.

We think of $X_s$ as a draw from some distribution with
\begin{equation}
\mu_s = \E[X_s]\in\R^d, \quad \Sigma_s = \var[X_s]\in \R^{d\times d}.
\label{mudef}
\end{equation}
The randomness in $X_s$ can be interpreted as reflecting the variation in the
characteristics for bank $s$ (and the macro variables) over time --- in particular, times
of stress. We assume throughout that each $\Sigma_s$ is nonsingular. The error
$\epsilon_s$ in (\ref{ys}) is assumed to satisfy, for each $s$,
\begin{equation}
\E[\epsilon_s]=0 \quad \mbox{ and } \quad \cov[X_s,\epsilon_s]=0.
\label{epcon}
\end{equation}

The regulator's problem is to choose a model $g$ that forecasts the loss rate $g(x,s)$
for bank $s$ if the bank's portfolio characteristic vector is $x$. The forecasts should,
at a minimum, satisfy the following narrow property, which prohibits the regulator
from applying different models to different banks:

\begin{definition}[Equal treatment]\label{d:et}
Model $g:\R^d\times\{1,\dots,\bar{S}\}\to \R$ satisfies
\emph{equal treatment} if $g(x,s)=g(x,s')$, for all
$x\in\R^d$, for all $s,s'\in\{1,\dots,\bar{S}\}$.
\end{definition}

As the true relationship for each bank is linear in (\ref{ys}), we mainly focus on the case
of a linear industry-wide model. The regulator's problem is then to choose a single
$\alpha\in\R$ and $\beta\in\R^d$ that it will use to form a forecast
\begin{equation}
\hat{Y}(x) = \alpha + \beta^{\top}x,
\label{yhat}
\end{equation}
given portfolio characteristics $x$. The forecast (\ref{yhat}) satisfies equal treatment
because it has no functional dependence on bank identity $s$. The parameters of the
industry model (\ref{yhat}) may depend on the bank-specific parameters
$(\alpha_s,\beta_s)$ and on the mean and variance in (\ref{mudef}), but they should
not depend on the realized features $X_s$.

The regulator would like the forecast loss $\hat{Y}(X_s)$ to be close to the actual loss
$Y_s$ in (\ref{ys}) for every bank $s$. To aggregate errors across banks, we introduce
a random variable $S$ that picks a bank according to a distribution
\begin{equation}
\PR(S=s) = p_s, \quad s=1,\dots,\bar{S},
\label{sdef}
\end{equation}
with the probabilities $p_s$ summing to 1. In the simplest case, all banks get equal
weight, and the $p_s$ are all equal; but the $p_s$ could also reflect relative asset sizes
or other weighting schemes. When we replace a bank label $s$ with the random
variable $S$, we get a mixture over banks. In particular, we can combine the
bank-specific models (\ref{ys}) into a mixture or hierarchical model by writing
\begin{equation}
Y_S = \alpha_S + \beta_S^{\top}X_S + \epsilon_S.
\label{ymix}
\end{equation}

In choosing parameters $\alpha$ and $\beta$ in (\ref{yhat}), the regulator would like
to make the forecast errors small for all banks. A natural way to aggregate forecast
errors across banks is to consider the average squared error, in which case the
regulator's problem becomes choosing $\alpha$ and $\beta$ in (\ref{yhat}) to solve
\begin{equation}
\min_{\alpha,\beta} \E[(\hat{Y}(X_S)-Y_S)^2].
\label{sqloss}
\end{equation}
The objective in (\ref{sqloss}) averages squared 
forecast errors over banks.
It can also be written as $\sum_sp_s\E[(\hat{Y}(X_s)-Y_s)^2]$.

\begin{remark}\label{r:formulation}
{\rm
Before solving (\ref{sqloss}), we make several comments on our problem formulation.

\noindent \textit{(i) Targets versus estimators.} The problem posed by (\ref{sqloss}), 
like the more general problem of choosing industry parameters in (\ref{yhat}), is one
of characterizing ideal coefficients $\alpha$ and $\beta$. This is a question of
choosing the correct {\it targets\/} of estimation, rather than a question of choosing
estimators. In particular, $\alpha$ and $\beta$ are population quantities rather than
sample quantities. In practice, the regulator may have a panel of time series of
observations across banks. Estimation methods for panel data ordinarily focus on
coefficients that are common to all units and exploit the panel structure to estimate
these shared values. Our concern is precisely with the case of heterogeneous
coefficients, where we need to identify suitable targets before we can consider their
estimation.

\noindent\textit{(ii) Groups versus individuals.} As discussed in Section~\ref{s:hf}, 
individual fairness is concerned with ensuring that if two feature vectors $x$ and $x'$
are close, then the predicted losses $\hat{Y}(x)$ and $\hat{Y}(x')$ are also close. Our
concern is for accuracy and fairness with respect to the distributions of $(X_s,Y_s)$,
$s=1,\dots,\bar{S}$, and not just for individual outcomes; we do not want the choice
of industry model to depend on the realization of $X_s$, $s=1,\dots,\bar{S}$. Each
$(X_s,Y_s)$ reflects a distribution over individual portfolio features and losses 
--- individuals that share the bank identity attribute $s$. Each bank thus represents
a group of potential individual portfolios, and we are interested in accuracy and
fairness with respect to the distributions that define these groups.

\noindent\textit{(iii) Stressed versus unstressed.} For the application to stress testing,
it is helpful to think of the portfolio features and scenario variables in $X_s$ and the
losses $Y_s$ in (\ref{ys}) as having their conditional distributions given stress
conditions. The regulator is then interested in forecasting the conditional mean loss
for each bank, given stress conditions. By focusing on the conditional mean, this
formulation makes the squared error (\ref{sqloss}) a reasonable benchmark for
studying accuracy and fairness. In our empirical work in Section~\ref{sec:experiment},
we approximate the restriction to stress conditions by giving more weight to data from
periods of stress in the squared error objective.

\noindent\textit{(iv) Regulator versus banks.} As discussed in Section~\ref{s:hf}, 
some of the objections raised by banks can be understood as pointing to features
missing from (\ref{ys}) and (\ref{yhat}), features that are rejected by the Fed as overly
narrow. Our investigation assumes the bank-specific models (\ref{ys})--(\ref{epcon})
are correct; in particular, under (\ref{epcon}) any relevant omitted features are
uncorrelated with included features. We focus on the regulator's problem of how best
to aggregate the bank-specific models (assuming their correctness) into an industry
model, considering both accuracy and fairness; we do not address the banks' claims
regarding which features should be included in the models.
} 
\end{remark}

For the solution to (\ref{sqloss}), write
\begin{equation}
\bar{\mu} = \E[X_S] = \E[\mu_S] = \sum_sp_s\mu_s \in \R^d,
\label{mubar}
\end{equation}
and
\begin{equation}
\var[X_S] = \E[(X_S-\bar{\mu})(X_S-\bar{\mu})^{\top}]
= \E[W_S] = \sum_s p_sW_s,
\label{vars}
\end{equation}
with
\begin{equation}
W_s = \Sigma_s + \mu_s\mu_s^{\top}-\bar{\mu}\mu_s^{\top}
\in\R^{d\times d}.
\label{wdef}
\end{equation}
Similarly,
$$
\cov[\alpha_S,\mu_S] = \sum_sp_s \alpha_s(\mu_s-\bar{\mu})\in\R^d.
$$

\begin{proposition}
Problem (\ref{sqloss}) is solved by
\begin{equation}
\beta_{Pool} = \E[W_S]^{-1}\left(\cov[\alpha_S,\mu_S] + \E[W_S\beta_S]\right)
\label{bpool}
\end{equation}
and
\begin{equation}
\alpha_{Pool} = \E[Y_S] - \beta_{Pool}^{\top}\bar{\mu}.
\label{apool}
\end{equation}
\label{p:pool}
\end{proposition}

Loss forecasts using $\alpha_{Pool}$ and $\beta_{Pool}$ in (\ref{yhat})
provide {\it fairness through unawareness}, in that they ignore bank identity.
They satisfy equal treatment in the narrow sense of Definition~\ref{d:et}.
Given our starting point (\ref{ys}), problem (\ref{sqloss}) would seem to be the
most direct interpretation of the Fed's policy of developing an ``industry-level
approach calibrated using data from many financial institutions.''

However, the solution in (\ref{bpool}) is not a satisfactory target. Indeed, (\ref{bpool})
shows where heterogeneity is most problematic. If the intercepts $\alpha_s$ covary
with the means $\mu_s$, this effect can distort $\beta_{Pool}$ through what
is commonly known as Simpson's paradox. As an extreme example, consider the
case that $\beta_s = 0$ for all $s$; in other words, none of the features in $X_s$
is predictive of losses for any of the banks. The regulator's model (\ref{yhat}) using
$\beta_{Pool}$ would nevertheless forecast losses based on these features if
$\cov[\alpha_S,\mu_S]$ is nonzero. This covariation would create the illusion of
predictability. In applying (\ref{bpool}), we would be forecasting losses based on
irrelevant features, purely as a consequence of the way we aggregated the
bank-specific models. 

Even in a less extreme setting in which the $\beta_s$ are nonzero, the presence of
the $\cov[\alpha_S,\mu_S]$ term in (\ref{bpool}) reflects an indirect influence of bank
identity on loss forecasts. If the bank-level mean characteristics $\mu_s$ positively
covary with the bank-level intercepts $\alpha_s$, then in the pooled model this
covariance will lead to a higher loss forecast for a bank with a higher value of $X_s$.
This is arguably unfair, in the sense that the loss forecast is not based on the
legitimate influence of the feature $X_s$. We will formalize the idea that the pooled
method misdirects legitimate information in Sections~\ref{s:feo} and~\ref{s:unified}.

This effect is reminiscent of the bias incurred in panel regressions when fixed effects
are present in the data but omitted from a model. As we emphasized in
Remark~\ref{r:formulation}(i), in our setting the primary objective is to define the
appropriate target of estimation, given the heterogeneity in the coefficients.
We cannot say the term $\cov[\alpha_S,\mu_S]$
introduces bias until we have decided what we are trying to estimate.

\subsection{Average Treatment Effects}
\label{s:scalar}

We can gain additional insight by considering the case of scalar $X_s$. In this case,
the pooled coefficient $\beta_{Pool}$ in (\ref{bpool}) becomes
\begin{equation}
\beta_{Pool} = 
\frac{\cov[\alpha_S,\mu_S] + \sum_sp_s(\sigma^2_s + \mu_s^2-\bar{\mu}\mu_s)\beta_s}{\sum_sp_s (\sigma^2_s + \mu_s^2-\bar{\mu}\mu_s)}.
\label{bpool1}
\end{equation}
In the special case that $\cov[\alpha_S,\mu_S]=0$ and 
$\sigma^2_s+\mu^2_s-\bar{\mu}\mu_s\ge 0$, for all $s$,
(\ref{bpool1}) becomes a convex combination of the individual $\beta_s$.
In Section~\ref{a:cxcomb}, we state some simple properties that 
an aggregation of the individual $\beta_s$ into a single industry value
should satisfy, and we show that only a convex combination satisfies
these properties. Equation (\ref{bpool1}) thus shows a further potential
problem with the pooled method. Even if $\cov[\alpha_S,\mu_S]=0$,
the coefficient on some $\beta_s$ could be negative, which would mean
that a reduction in $\beta_s$ would increase $\beta_{Pool}$. This
could mean that an improvement in risk management by one bank
\textit{increases} predicted losses at all banks. We investigate these
types of cross-bank effects further in Section~\ref{s:cross}.

We will refer to any convex combination of the $\beta_s$ as a  {\it weighted
average treatment effect} or WATE parameter. This terminology is suggested
by thinking of a unit increase in a portfolio characteristic $X_s$ as a
treatment, and $\beta_s$ as the response to that treatment. The (ordinary)
average treatment effect is the expected coefficient,
\begin{equation}
\beta_\textit{ATE}=\E[\beta_S] = \sum_sp_s\beta_s,
\label{ate}
\end{equation}
but weighting the individual coefficients allows other combinations. In
particular, if the $\mu_s$ are all equal, the pooled coefficient (\ref{bpool1})
becomes
\begin{equation}
\beta_{Pool} = 
\frac{\sum_sp_s\sigma^2_s\beta_s}{\sum_sp_s \sigma^2_s}.
\label{bpf}
\end{equation}
We will say more about these cases in subsequent sections.

To translate a WATE coefficient into a loss projection $\hat{Y}$, as in
(\ref{yhat}), we also need to specify an intercept. Setting
$$
\alpha_\textit{WATE} = \E[Y_S] - \beta^{\top}_\textit{WATE}\bar{\mu},
$$
ensures that the forecasts
$$
\hat{Y}_\textit{WATE}(X_s) = \alpha_\textit{WATE} + \beta^{\top}_\textit{WATE}X_s,
\quad s=1,\dots,\bar{S},
$$
have zero expected error, in the sense that
$$
\E[\hat{Y}_\textit{WATE}(X_S) - Y_S]
= \sum_sp_s(\alpha_\textit{WATE} +\beta^{\top}_\textit{WATE}\mu_s) - \E[Y_S]=0.
$$

\section{Fair Regressions}
\label{s:freg}

We have seen that if the regulator's sole objective is to minimize average
squared forecast errors subject to equal treatment, then the solution is
given by the pooled coefficients in (\ref{bpool}) and (\ref{apool}).
However, we have also seen that (\ref{bpool}) has consequences that are
undesirable and even unfair, in the sense that it is indirectly influenced
by bank identity. In this section, we turn to methods that expand the
squared loss minimization objective (\ref{sqloss}) to include fairness
considerations. Because the pooled method minimizes (\ref{sqloss}), any
method that addresses fairness will entail a loss of accuracy as measured
by (\ref{sqloss}).

\subsection{Projection to Fairness}
\label{s:ptf}

In the literature on fairness in classification methods, {\it demographic
parity} is among the most widely discussed fairness principles; see, for
example, Chapter 3 of Barocas et al.\ \cite{fairmlbook}. In the simplest
classification setting, the counterpart of our forecast is a binary outcome
$\hat{Y}\in \{0,1\}$. For example, $\hat{Y}=1$ may indicate a hiring
decision, a loan approval, or a school admission decision. The decision
is to be based on certain features of a candidate that are deemed
legitimate. Demographic parity requires that the event $\{\hat{Y}=1\}$
be statistically independent of a protected attribute, such as race or
gender. This objective is difficult to achieve when legitimate features
covary with the protected attribute.

Chzhen et al.\ \cite{chzhen} and Le Gouic et al.\ \cite{legouic} extend the
notion of demographic parity to the regression setting by requiring that
model predictions be independent of a protected attribute. These two
articles solve the problem of finding the model that minimizes mean
squared prediction errors while achieving demographic parity. We will use
the term {\it projection to fairness} (PTF), coined in Le Gouic et al.\ 
\cite{legouic}, for the method in these papers.

Both papers reduce the problem of regression fairness to one of finding
the Wasserstein barycenter of a set of distributions, in the sense of Agueh
and Carlier \cite{agueh}. The barycenter is the distribution closest to the
set of distributions in an average sense. For a squared error and
one-dimensional distributions, the barycenter can be described as the
distribution whose quantile function is a weighted average of the individual
quantile functions. (The quantile function is the inverse of the cumulative
distribution function.)

In the setting of Section~\ref{s:model}, the resulting solution can be
interpreted as follows. For bank $s$, the regulator first forms the forecast
$\hat{Y}_s(x) = \alpha_s+\beta^{\top}_sx$, using the bank-specific
coefficients and the realized features $X_s=x$. Suppose $\hat{Y}_s$ falls
at the 80th percentile of the forecast distribution for bank $s$. The
regulator then takes a weighted average of the 80th percentile forecast for
all of the bank-specific models. That weighted average becomes the
forecast for bank $s$.

To make this procedure more explicit and to specialize the general
framework of Chzhen et al.\ \cite{chzhen} and Le Gouic et al.\ \cite{legouic}
to our setting, we consider the case (for this section only) that each feature
vector $X_s$ has a multivariate normal distribution $N(\mu_s,\Sigma_s)$.
Write $\Sigma_s^{1/2}$ for the symmetric square root of $\Sigma_s$,
and define the standardized feature vectors
\begin{equation}
Z_s = \Sigma_s^{-1/2}(X_s-\mu_s);
\label{zdef}
\end{equation}
each $Z_s$ has a multivariate standard normal distribution. Write the
basic identity (\ref{ys}) using standardized variables as
$$
Y_s = \alpha^o_s + \beta_s^{o\top}Z_s + \epsilon_s,
$$
with standardized coefficients
\begin{equation}
\beta^o_s = \Sigma^{1/2}_s\beta_s, \quad 
\alpha_s^o = \alpha_s + \beta_s^{\top}\mu_s.
\label{stab}
\end{equation}
Suppose $\|\beta^o_s\|\not= 0$, for all $s$, 
with $\|\cdot\|$ denoting the usual Euclidean norm.
Consider the model that assigns, to each bank $s=1,\dots,\bar{S}$, 
with features $X_s=x$ the forecast
\begin{equation}
\hat{Y}^o(x,s) = \sum_ip_i\alpha^o_i + 
\sum_i p_i\|\beta_i^o\|\frac{\beta_s^{o\top}z_s}{\|\beta_s^o\|}, \quad z_s = \Sigma_s^{-1/2}(x-\mu_s).
\label{ptfn}
\end{equation}
If there exists a $\beta\in\mathbb{R}^d$ and scalars $a_s>0$ for which
\begin{equation}
\beta_s^o = a_s\beta, \quad s=1,\dots,\bar{S},
\label{bprop}
\end{equation}
then we will see that (\ref{ptfn}) simplifies to the weighted average
\begin{equation}
\hat{Y}^o(x,s) = \bar{\alpha}^o + \bar{\beta}^{o\top}z_s, \quad
\bar{\alpha}^o= \sum_ip_i\alpha^o_i, \; \bar{\beta}^o= \sum_i p_i\beta_i^o.
\label{ptfs}
\end{equation}
In the case of scalar $X_s$, 
(\ref{bprop}) holds whenever all $\beta_s$ have the same sign.

\begin{proposition}
Suppose that the $X_s$ are multivariate normal and $\|\beta_s\|\not=0$,
for all $s=1,\dots,\bar{S}$. Then (\ref{ptfn}) is the projection-to-fairness
of the bank-specific models (\ref{ys}), meaning that (\ref{ptfn}) minimizes
$\E[(\hat{Y}^o(X_S,S)-Y_S)^2]$ among all models (whether linear or not)
that satisfy demographic parity. If (\ref{bprop}) holds, the
projection-to-fairness is given by (\ref{ptfs}).
\label{p:ptf}
\end{proposition}

We can see from (\ref{ptfn}) that the PTF model does not satisfy equal
treatment: to calculate the loss forecast for a bank, we need to know its
identity $s$. We have included the special case of (\ref{ptfs}) because it
more nearly parallels the type of model we seek in (\ref{yhat}). The
coefficients in (\ref{ptfs}) are weighted averages of bank-specific
coefficients. The model in (\ref{ptfs}) satisfies equal treatment with
respect to the standardized features $Z_s$, rather than the raw features
$X_s$: two banks with the same standardized features will receive the
same forecasts. But the means for the two banks could be very different
 --- the standardization is done separately for each bank --- indicating
that one bank's portfolio may be much riskier than the other bank's. In
treating standardized characteristics for different banks as comparable,
the PTF model implicitly evaluates the riskiness of each bank relative to
the distribution for that bank. The suitability of PTF in our setting is
therefore questionable.

The root of the problem is that demographic parity is too strong a property
for our setting. Ensuring that a hiring decision is independent of race or
gender is important; but forcing the distribution of loss projections to be
independent of bank identity ignores relevant differences in banks' portfolios.
Whereas the pooled model (\ref{bpool})--(\ref{apool}) does too little to
address heterogeneity across banks, the PTF model goes too far in leveling
differences. The next section provides a better balance.

\subsection{Formal Equality of Opportunity}
\label{s:feo}

Johnson, Foster, and Stine \cite{jfs} introduce the concept of formal equality
of opportunity (FEO) in regression, based on the use of the term in political
philosophy, for which they cite the review in Arneson \cite{arneson}.
According to Arneson \cite{arneson}, FEO means that ``positions and posts
that confer superior advantages should be open to all applicants. Applications
are assessed on their merits.''

In adapting this idea to our setting, it is helpful to make a contrast with the
previous section: whereas demographic parity requires that loss forecasts be
independent of bank identity, FEO allows bank-dependence, but only through
legitimate portfolio characteristics --- through the bank's ``merits.'' This
notion aligns well with the Fed policy, quoted earlier, that ``two firms with
the same portfolio receive the same results.'' The objective of FEO in regression,
as developed by Johnson et al.\ \cite{jfs}, is to ensure that a protected attribute
(for us, bank identity) has no direct or ``causal'' impact on a model's predictions. 
The predictions may be correlated with bank identity if different banks tend to
have different levels of exposure to legitimate portfolio features.

To develop this idea in our setting, we introduce the centered dummy variables
\begin{equation}
U_i(s) = \mathbf{1}\{s=i\} - p_i, \quad i=1,\dots,\bar{S}-1,\;
s=1,\dots,\bar{S}.
\label{usdef}
\end{equation}
We discuss the implications of centering below.
For any coefficients $\alpha,\delta_1,\dots\delta_{\bar{S}-1}\in\R$
and $\beta\in\R^d$, and any $x\in\R^d$, let
\begin{equation}
\hat{Y}(x,s) = \alpha + \sum_i \delta_i U_i(s) + \beta^{\top}x.
\label{ytilde}
\end{equation}
We have included the bank label $s$ as an argument of $\hat{Y}$ 
because $U_i$ depends on $s$.
Let $\alpha_F$, $\{\delta_i$, $i=1,\dots,\bar{S}-1\}$,
and $\beta_F$ solve the error minimization problem
\begin{equation}
\min_{\alpha,\{\delta_i\},\beta}\E[(\hat{Y}(X_S,S)-Y_S)^2].
\label{sqlossf}
\end{equation}
With the coefficients that minimize (\ref{sqloss}), (\ref{ytilde}) becomes
the linear projection of $Y_S$ onto the span of 
$\{1, U_1(S),\dots,U_{\bar{S}-1}(S),X_S\}$,
evaluated at $S=s$ and $X_S=x$.
Now drop the centered dummy variables $U_i$ and define
\begin{equation}
\hat{Y}_F(x) = \alpha_F +  \beta_F^{\top}x.
\label{yfeo}
\end{equation}
The FEO loss forecast for bank $s$ is $\hat{Y}_F(X_s)$.

Steps (\ref{ytilde})--(\ref{yfeo}) result from applying the definition of an
impartial estimate (their Definition 2) in Johnson et al.\ \cite{jfs}. (More
precisely, steps (\ref{ytilde})--(\ref{yfeo}) define a population counterpart
of the sample formulation in \cite{jfs}.) The procedure in
(\ref{ytilde})--(\ref{yfeo}) can be interpreted as follows: pool losses and
portfolio features across banks; regress losses on portfolio features with
bank fixed-effects included; throw away the fixed effects in forecasting
future losses. The resulting model (\ref{yfeo}) is an equal-treatment model,
with no explicit dependence on bank identity. Centering the discarded
variables $U_i$ ensures that $\E[\hat{Y}_F(X_S)]=\E[Y_S]$, so dropping the
fixed effects does not introduce an overall bias.

We will say more about the implications of this approach, but we first show
that our setting allows an explicit expression for the FEO coefficients:

\begin{proposition}\label{p:feo}
(i) The FEO coefficients are given by
\begin{equation}
\beta_F = \E[\Sigma_S]^{-1} \E[\Sigma_S\beta_S],
\label{bfeo}
\end{equation}
and
\begin{equation}
\alpha_F = \E[Y_S] - \beta_F^{\top}\bar{\mu}.
\label{afeo}
\end{equation}
In particular, in the scalar case,
\begin{equation}
\beta_F = \frac{\sum_sp_s\sigma^2_s\beta_s}{\sum_sp_s\sigma^2_s}.
\label{bfeo1}
\end{equation}
(ii) We also have
\begin{equation}
\beta_F = \var[X_S-\mu_S]^{-1}\cov[X_S-\mu_S,Y_S],
\label{bfw}
\end{equation}
so $\beta_F^{\top}(X_S-\mu_S)$ is the linear projection of $Y_S-\E[Y_S]$
onto $X_S-\mu_S$.
\end{proposition}

We encountered (\ref{bfeo1}) in (\ref{bpf}) as a special case of the pooled
coefficient when the bank means $\mu_s$ are constant. The general case
in (\ref{bfeo}) similarly coincides with the pooled coefficient in (\ref{bpool})
when the means are constant. In other words, introducing the bank-level
fixed effects in (\ref{ytilde}) purges $\beta_F$ of the effect of different
feature means across banks; dropping these fixed effects in (\ref{yfeo})
ensures that the regulator's model has no explicit dependence on bank
identity and satisfies equal treatment.

In what sense is this procedure fair? We adapt the interpretation in Johnson
et al.\ \cite{jfs} to our setting. Write $U = (U_1,\dots,U_{\bar{S}-1})^{\top}$
for the vector of centered dummy variables. Write $\cov[X_S,U(S)]$ for the
$d\times(\bar{S}-1)$ matrix of covariances between the components of
$X_S$ and $U(S)$. Let
\begin{equation}
\Lambda = (\var[X_S])^{-1}\cov[X_S,U(S)].
\label{lamdef}
\end{equation}
This matrix minimizes $\E[\|U(S) - \Lambda^{\top}(X_S-\bar{\mu})\|^2]$, 
so $\Lambda^{\top}(X_S-\bar{\mu})$ is the linear projection of the
bank-identity variables $U(S)$ onto the centered portfolio features
$X_S-\bar{\mu}$. The relationship between $\beta_{Pool}$ and $\beta_F$
can be expressed as follows. 

\begin{proposition}\label{p:bbld}
The coefficients $\beta_{Pool}$ and $\beta_F$ satisfy
\begin{equation}
\beta_{Pool} = \beta_{F} + \Lambda\delta,
\label{ovb}
\end{equation}
where  $\delta = (\delta_1,\dots,\delta_{\bar{S}-1})^{\top}$ is the vector
of coefficients from (\ref{ytilde})--(\ref{sqlossf}). In particular,
\begin{equation}
\delta_s = (\alpha_s + \beta_s\mu_s) - (\alpha_{\bar{S}}+\beta_{\bar{S}}\mu_{\bar{S}})
-\beta^{\top}_F(\mu_s - \mu_{\bar{S}}),
\quad s=1,\dots,\bar{S}-1.
\label{delta}
\end{equation}
\end{proposition}

We can write the forecast in (\ref{ytilde}), using the optimal coefficients
from (\ref{sqlossf}) as
\begin{equation}
\hat{Y}(x,s) = \E[Y_S] + \delta^{\top}U(s) + \beta_F^{\top}(x-\bar{\mu});
\label{ylp}
\end{equation}
This is the linear projection of $Y_S$ onto $(1,U(S),X_S)$, evaluated at
$S=s$, $X_s=x$. Let $\hat{Y}_P(x) = \alpha_{Pool}+\beta^{\top}_{Pool}x$
denote the forecast based on the pooled coefficients (\ref{bpool}) and
(\ref{apool}). Decomposing $U(S)$ into its projection onto $X_S-\bar{\mu}$
and an orthogonal component leads to the following contrast of these
forecasts:
\begin{align}
\hat{Y}(x,s) &= \E[Y_S] + \delta^{\top}\Lambda^{\top} (x-\bar{\mu}) +
\delta^{\top}[U(s)-\Lambda^{\top} (x-\bar{\mu})] \hspace*{-.9in}
& + \beta_F^{\top}(x-\bar{\mu})
\label{cytilde} \\
\hat{Y}_P(x) &= \E[Y_S] + \delta^{\top}\Lambda^{\top} (x-\bar{\mu})  
& + \beta_F^{\top}(x-\bar{\mu})
\label{cypool} \\
\hat{Y}_F(x) &= \E[Y_S]   &+ \beta_F^{\top}(x-\bar{\mu})
\label{cyfeo}
\end{align}

The term $\delta^{\top}[U(s)-\Lambda^{\top} (x-\bar{\mu})]$ in (\ref{cytilde})
affects the forecast through
information in bank identity that is orthogonal to the legitimate features $x$.  
This would be {\it disparate treatment}, as in Johnson et al.\ \cite{jfs}. 
Through ``unawareness'' (meaning that it has no functional dependence on
bank identity) the pooled forecast (\ref{cypool}) drops this term, but it retains
$\delta^{\top}\Lambda^{\top}(x-\bar{\mu})$, as can be seen from (\ref{ovb}).

The term $\delta^{\top}\Lambda^{\top}(x-\bar{\mu})$ is the problematic
component of the pooled method. Although it does not explicitly use bank
identity, this term relies on the fact that bank identity is to some extent
predictable from portfolio features. Imagine the regulator forming loss
forecasts from blinded data --- the regulator does not know the identity of
the bank. The term $\Lambda^{\top}(x-\bar{\mu})$ is the least-squares
prediction of $U(s)$ from $x-\bar{\mu}$. In the pooled forecast (\ref{cypool}),
the regulator is implicitly ``misdirecting'' the data in the features $x-\bar{\mu}$
to try to identify the bank and then to adjust the forecast based on the inferred
identity. The FEO forecast (\ref{cyfeo}) removes this effect and retains only the
direct effect of portfolio features on the loss rate.

In the terminology of Section~\ref{s:hf}, dropping 
$\delta^{\top}[U(s)-\Lambda^{\top} (x-\bar{\mu})]$ ensures the narrow sense
of equal treatment --- that loss forecasts not depend explicitly on bank identity.
Dropping $\Lambda^{\top}(x-\bar{\mu})$ ensures a broader sense of equal
treatment --- that loss forecasts not depend on proxies for bank identity.
Johnson et al.\ \cite{jfs} refer to their counterpart of
$\Lambda^{\top}(x-\bar{\mu})$ as {\it disparate impact}, which is consistent
with the notion of ``proxy discrimination'' as a particular type of disparate
impact (as in Prince and Schwarcz \cite{prince}). In our setting, as noted in
Section~\ref{s:hf}, the disparate impact of most immediate concern to banks is
the omission of features from the Fed's models that might otherwise benefit
individual banks. Omitted features may contribute to
$\Lambda^{\top}(x-\bar{\mu})$, but dropping this term does not necessarily
dispel banks' complaints. The banks' disagreements with the Fed concern the
scope of portfolio features that should be modeled. We therefore prefer to
associate $\delta^{\top}[U(s)-\Lambda^{\top} (x-\bar{\mu})]$ and
$\Lambda^{\top}(x-\bar{\mu})$ with narrow and broad interpretations of the
Fed's own principle of equal treatment (as discussed in Section~\ref{s:hf}), rather
than with separate concerns for disparate treatment and disparate impact by
the Fed and the banks. We emphasize the interpretation of
$\Lambda^{\top}(x-\bar{\mu})$ as a misdirection of legitimate information, 
rather than as a contributor to disparate impact.

The FEO method offers a further advantage over the pooled method.
Recall again from Section~\ref{s:hf} (and Remark~\ref{r:formulation}(iv)) 
that we interpret the banks' objections as calls for the inclusion
of features that are omitted from the Fed's models. Under the condition
$\cov[X_s,\epsilon_s]=0$ in (\ref{epcon}), the FEO coefficients of included
features are unaffected by the omission of other features. The pooled
coefficients do not in general have this property.

We will conclude in Section~\ref{s:unified} that the FEO forecast is, in a precise
sense, the best way to aggregate the bank-specific models into a single
regulatory model. The FEO forecast has no direct dependence on bank identity;
but it also removes the indirect dependence that results when bank identity is
partly predictable from portfolio features. We discuss other methods for
comparison.

\subsection{Conditional Expectation Model}
\label{s:cond}

A similar misdirection of information occurs if we project the bank-specific
models to an industry model in the sense of conditional expectation, rather
than least squares. Suppose $X_s$ has density $g_s$, and suppose
$\E[\epsilon_s|X_s]=0$, $s=1,\dots,\bar{S}$.
Then, by Bayes' rule,
\begin{equation}
\hat{Y}_C(x) \equiv \E[Y_S|X_S=x] = \frac{\sum_sp_s g_s(x)(\alpha_s+\beta_s^{\top}x)}{\sum_sp_sg_s(x)}.
\label{ycon}
\end{equation}
This model satisfies equal treatment --- $\hat{Y}_C(x)$ depends on the
portfolio features $x$ but not on a bank's identity. However, the point of
the weights $p_sg_s(x)$ is to infer the identity of the bank from the features.
Indeed, as discussed in Section~\ref{s:nonlinear}, the conditional expectation
$\E[Y_S|X_S=x]$ can be viewed as a nonlinear generalization of the pooled
method, with some of the same shortcomings.

\subsection{Substantive Equality of Opportunity}
\label{s:seo}

As discussed in Arneson \cite{arneson}, a system in which admission
decisions are made through a competitive exam open to everyone achieves
formal equality of opportunity; but if only the wealthy have access to the
preparation required for the exam, the system fails to achieve
{\it substantive\/} equality of opportunity (SEO). In the regression setting,
Johnson et al.\ \cite{jfs} interpret SEO to mean that any influence of
protected attributes should be removed from other variables included in a
regression model. In the analogy with Arneson's \cite{arneson} example, 
SEO would seek to remove the effect of economic status from performance
on the exam, whereas FEO would accept exam scores as a legitimate basis
for decision-making.
(Our use of ``SEO'' follows Johnson et al.\ \cite{jfs}. For a broader interpretation
of substantive equality in algorithmic fairness, see Green \cite{green}.)

To apply these ideas to our setting, define the $(\bar{S}-1)\times d$ matrix
\begin{equation}
M = \var[U(S)]^{-1}\cov[U(S),X_S];
\label{Mdef}
\end{equation}
then $M$ minimizes $\E[\|X_S-\bar{\mu} - M^{\top}U(S)\|^2]$.
In accordance
with Definition 2 of Johnson et al.\ \cite{jfs}, define
\begin{equation}
\hat{Y}_{SEO}(x,s) = 
\alpha_F + \beta_F^{\top}(x - M^{\top}U(s)),
\label{yseo}
\end{equation}
with $\alpha_F$ and $\beta_F$ defined by (\ref{sqlossf}). The SEO forecast
adjusts the portfolio features $x$ to remove the linear projection onto the
centered bank dummy variables $U$. We can write (\ref{yseo}) somewhat
more explicitly as follows:

\begin{proposition}\label{p:seo}
With $M$ as in (\ref{Mdef})
\begin{equation}
M^{\top}U(s) = \sum_i(\mu_i-\mu_{\bar{S}})U_i(s) =\mu_s-\bar{\mu}, 
\label{mtu}
\end{equation}
so the SEO forecast (\ref{yseo}) is given by
\begin{equation}
\hat{Y}_{SEO}(x,s) = 
\alpha_F + \beta_F^{\top}(x - \mu_s+\bar{\mu}).
\label{yseo2}
\end{equation}
The SEO forecast is the linear projection of $Y_S$ onto a constant and
$X_S-\mu_S$.
\end{proposition}

Recall from Section~\ref{s:ptf} that a model satisfies demographic parity
if its forecasts are independent of bank identity. Let us say that a model
satisfies {\it weak\/} demographic parity if its forecasts are
{\it uncorrelated\/} with the bank-identity variables $U_i(S)$. The centered
features $X_S-\mu_S$ are uncorrelated with the $U_i(S)$. It therefore
follows from Proposition~\ref{p:seo} that SEO forecasts are uncorrelated
with the $U_i(S)$. In other words, we have the following result:

\begin{corollary}\label{c:wdp}
The SEO forecast satisfies weak demographic parity.
\end{corollary}

Under additional conditions, we get a stronger conclusion:

\begin{corollary}
If the covariance matrix $\Sigma_s$ and the distribution of $Z_s$ in
(\ref{zdef}) are the same for all $s$, then the SEO model coincides with the
standardized model (\ref{ptfs}), and both satisfy demographic parity.
\label{c:ptfseo}
\end{corollary}

Under the conditions in the corollary, the mean adjustment in (\ref{yseo2}) 
is sufficient to give $\hat{Y}_{SEO}(X_s,s)$ the same distribution for all $s$.
Put differently, PTF considers only the quantile of
$\alpha_s+\beta_s^{\top}X_s$, relative to the distribution for bank $s$,
to be legitimate information; SEO considers $X_s-\mu_s$ to be legitimate
information. Under the conditions of the corollary, the two concepts coincide.

The mean adjustment in (\ref{yseo2}) requires knowledge of the bank identity
$s$, so (\ref{yseo}) does not satisfy Definition~\ref{d:et}.
The intent of the mean adjustment is to achieve a greater degree of equality. Consider the example
with which began this section. If $x$ represents an exam score and
$\mu_1>\mu_0$ are the mean scores among wealthy and non-wealthy exam
takers, (\ref{yseo2}) adjusts scores downward for wealthy exam takers and
upward for non-wealthy exam takers.

Such an adjustment may be appropriate when the individuals or firms under
evaluation are, in some sense, not responsible for their mean characteristic
(or the mean in their peer group) and are therefore evaluated based on
deviations from the mean. This type of consideration does not seem
applicable to the stress-test setting, but it could arise more generally in
settings where capital regulation intersects with other policy objectives.

One such example is suggested by the Paycheck Protection Program Lending
Facility (PPPL) launched by the Federal Reserve early in the COVID crisis. The
PPPL provided for loans to small businesses to be made by banks and
guaranteed by the Small Business Administration. Under normal
circumstances, the loans would increase participating banks' balance sheets
and thus potentially increase their capital requirements. To promote use
of the facility, banking regulators issued a rule excluding PPPL loans from
capital requirements, thus ``neutralizing the effects of participating in the
PPPL Facility on regulatory capital requirements.''%
\footnote{Federal Register, Vol.~85, No.~71, p.20389, April 13, 2020.}
This ``neutralizing'' action is somewhat analogous to the SEO adjustment
in that it removes responsibility for the larger balance sheet from the bank.
The adjustments differ in that SEO adjusts for the mean whereas the PPPL
adjustment removes the amount lent through the program. 

\subsection{A Unified Perspective: Legitimate Information}
\label{s:unified}

All of the methods we have discussed can be seen as ways of choosing
forecasts $\hat{Y}_s$, $s=1,\dots\bar{S}$, (of the form $\hat{Y}(X_s)$
or $\hat{Y}(X_s,s)$) to minimize
\begin{equation}
\E[(\hat{Y}_S-Y_S)^2],
\label{sqloss2}
\end{equation}
subject to additional considerations. Table~\ref{t:uni} summarizes the cases
we have considered. In rows (i), (iv), and (v), we minimize (\ref{sqloss2}) over
the indicated coefficients. In (ii) and (iii), we allow $g$ to be an arbitrary
(suitably measurable) function of the indicated arguments. In (iii) we
strengthen the condition (\ref{epcon}) on the errors $\epsilon_s$.

\begin{table}
\begin{tabular}{clll}
		&Form                                                       & Constraint           & Forecast \\ \hline
(i)   &	$\hat{Y}_s = \alpha + \beta^{\top}X_s$    &                           & Pooled (\ref{bpool})--(\ref{apool})  \rule{0pt}{12pt}\\
(ii)  &	$\hat{Y}_s = g(X_s,s)$, some $g$              & $\hat{Y}_S$ independent of $S$ & PTF (\cite{chzhen,legouic}) \\ 
(iii)  &	$\hat{Y}_s = g(X_s)$, some $g$, $\E[\epsilon_s|X_s]=0$&       & Cond.~exp. (\ref{ycon}) \\
(iv) &	$\hat{Y}_s = \alpha + \beta^{\top}X_s$ & $\cov[Y_S-\hat{Y}_S,X_S-\mu_S]=0$ & FEO (\ref{yfeo}) \\
(v) &   $\hat{Y}_s = \alpha + \lambda^{\top}U(s) + \beta^{\top}X_s$ & $\cov[\hat{Y}_S,U(S)]=0$ & SEO (\ref{yseo}) \\
\hline
\end{tabular}
\caption{Summary of forecast model forms and constraints.}
\label{t:uni}
\end{table}

\begin{proposition}
In each row of Table~\ref{t:uni}, the squared loss (\ref{sqloss2}) is minimized
over forecasts of the form in the first column, subject to the constraint in the
second column, by the model in the last column.
\label{p:unified}
\end{proposition}

The constraint in Table~\ref{t:uni}(v) is weak demographic parity. SEO implicitly
takes the view that the only legitimate information in forecasting losses for
bank $s$ is the deviation $X_s-\mu_s$.
In contrast, FEO takes the full set of features $X_s$ as legitimate information.
Through the constraint in Table~\ref{t:uni}(iv), it enforces a requirement we
call \emph{no misdirection of legitimate information}. FEO uses all of $X_s$ in
forecasting losses; but it chooses the coefficient $\beta_F$ to be the
coefficient in a regression of $Y_S$ on $X_S-\mu_S$, which is the part of
$X_S$ orthogonal to bank identity. This condition ensures that the
information in $X_S$ is not misdirected to infer bank identity.

To make this idea precise, consider any model of the form (\ref{yhat}). If
we assume the intercept is chosen to match the unconditional mean, we
may write the model as
\begin{equation}
\hat{Y}_{\gamma}(x) = \E[Y_S] + (\beta_F+\gamma)^{\top}(x-\bar{\mu}),
\label{bg}
\end{equation}
for some $\gamma\in\mathbb{R}^d$. With $\gamma=0$, we get the FEO
forecast (\ref{yfeo}).

\begin{proposition}\label{p:nmli}
If $\gamma$ reduces errors in the sense that
$\E[(\hat{Y}_{\gamma}(X_S)-Y_S)^2]<\E[(\hat{Y}_F(X_S)-Y_S)^2]$,
then the forecast
$\hat{Y}_{\gamma}$ misdirects legitimate information in the sense that
\begin{itemize}
\item[(i)] 
$\cov[\gamma^{\top}X_S,\delta^{\top}\Lambda^{\top}X_S]>0$, and
\item[(ii)] 
$\cov[\gamma^{\top}M^{\top}U(S),\delta^{\top}U(S)]>0$.
\end{itemize}
\end{proposition}

Recall that $\Lambda^{\top}(X_S-\bar{\mu})$ is the linear projection
of the centered bank identity variables $U(S)$ onto the centered
portfolio features $X_S-\bar{\mu}$. The condition in (i) therefore
indicates that $\gamma$ misdirects some of the legitimate information
in $X_S$ toward inferring bank identity. Thus, deviating from 
$\beta_F$ in (\ref{bg}) either increases errors or misdirects information.

Property (ii) has a similar interpretation. The term $\delta^{\top}U(S)$
is the direct influence of bank identity on losses $Y_S$. The proposition
states that any deviation $\gamma$ that reduces forecast errors
(relative to $\gamma=0$) implicitly picks up some of the information
in bank identity.

To further illustrate the contrast between FEO and SEO
consider a simple example in which some
component of $X_s$ measures exposure to community development
projects. Suppose for simplicity that this feature is uncorrelated with
other features. In the SEO forecast, the only legitimate information
from this exposure is a bank's deviation from its own mean. Years in
which a bank had above average exposure would lead to higher loss
forecasts, but the bank's average exposure to community development
would not directly inform the forecasts --- it is neutralized. In contrast,
FEO treats the bank's total exposure (mean plus deviation) as legitimate
information. Like SEO, in evaluating the impact of this exposure --- that
is, in estimating the coefficient on the exposure --- it relies only on the
within-bank variation. This ensures that the information in the exposure
is not misdirected toward inferring the bank's identity, as could happen
in the pooled regression.

\subsection{Extension of FEO for Interaction Effects}
\label{s:atefeo}

Recall that the FEO forecast controls for bank fixed effects. One might
similarly consider controlling for interactions between bank indicators
and components of the feature vectors. This leads to a family of
extensions of FEO that differ in which interactions they include. We will
show that with a full set of interactions, the extended FEO model becomes
the ATE model (\ref{ate}).

To examine this case, suppose the feature vector for each bank $s$ is
partitioned into two components, $X_s$ and $V_s$. We extend FEO by including
interactions with components of $V_s$ but not with components of $X_s$.
(Thus, in our discussion of FEO, $V_s$ was empty.) We assume that for
every bank $s$, the components of $X_s$ are uncorrelated with the
components of $V_s$. This allows a clear delineation between variables
with and without interactions. Let $\nu_s=\E[V_s]$. The bank-specific
models (\ref{ys}) now take the form
\begin{equation}
Y_s = \alpha_s + \beta_s^{\top}X_s + \gamma_s^{\top}V_s + \epsilon_s,
\label{ysz}
\end{equation}
with $\epsilon_s$ uncorrelated with $X_s$ and $V_s$.

We extend FEO to the following procedure:
\begin{itemize}
\item[1)] Project $Y_S$ linearly onto 1, $U_1(S),\dots,U_{\bar{S}-1}(S)$,
$X_S-\mu_S$, $V_S-\nu_S$, $U_1(S)V_S$, $\dots$, $U_{\bar{S}}(S)V_S$.
Let $\beta_F$ denote the coefficient of $X_S-\mu_S$ and let $\gamma_F$
denote the coefficient of $V_S-\nu_S$.
\item[2)] Set $\hat{Y}_F(x,v) = \alpha_F + \beta_F^{\top}x + \gamma_F^{\top}v$,
with $\alpha_F$ chosen so that $\E[\hat{Y}(X_S,V_S)]=\E[Y_S]$.
\end{itemize}
If $V_S$ is empty, then we know from (\ref{bfw}) that these steps do
indeed reduce to the original FEO forecast. We have included the interaction
$U_{\bar{S}}(S)V_S$ in the first step (even though we omitted $U_{\bar{S}}(S)$)
to simplify the derivation of $\gamma_F$. Including this term means that
the coefficients on the interactions $U_i(S)V_S$ are determined only up to 
constant, because $U_1(S)V_S + \cdots +U_{\bar{S}}(S)V_S=0$. These
coefficients are dropped in the second step, so their value is immaterial.

\begin{proposition}\label{p:ate}
Suppose $\var[X_s]$ and $\var[V_s]$ have full rank and $X_s$ and $V_s$
are uncorrelated, for each $s=1,\dots,\bar{S}$. Then $\beta_F$ is given
by (\ref{bfeo}) and (\ref{bfw}), and $\gamma_F = \bar{\gamma} =
\sum_sp_s\gamma_s$. In particular, if interactions with $U(S)$ are
included for all features, the FEO vector of coefficients reduces to the
average treatment effect (\ref{ate}).
\end{proposition}

This result allows us to interpret the ATE forecast as a version of the FEO 
forecast that removes the effects of certain interactions. As a convex
combination of the bank-specific coefficients, the ATE coefficient retains
some of the advantages of the FEO coefficient, particularly for the
cross-bank effects studied in Section~\ref{s:cross}.

However, we do not see a compelling case for controlling for interactions
between bank identity and portfolio features. When we control for the
bank-identity variables in FEO, we are ensuring that the industry $\beta$
for legitimate features is not affected by heterogeneity in the banks'
constants (the fixed effects). This reasoning does not necessarily extend
to removing the influence of heterogeneity in exposures to portfolio features.

\section{Nonlinear Models}
\label{s:nonlinear}

Most of the ideas developed in previous sections for linear regressions extend to generalized
linear models through a transformation of the response variable. For example, instead of
working with the loss rate $Y_s$, we could specify a linear model for its logit transformation
$\log(Y_s/(1-Y_s))$.

But we can also extend ideas from previous sections to more fully nonlinear models.
Replace the mixture model in (\ref{ymix}) with a general representation of the form
\begin{equation}
Y_S = g(S,X_S) + \epsilon_S, \quad \E[\epsilon_S|S,X_S]=0.
\label{ynon}
\end{equation}
In other words, the loss for bank $s$ is given by $g(s,X_s)+\epsilon_s$.
We assume that $g(S,X_S)$ and $\epsilon_S$ are square-integrable.

The counterpart of the pooled estimate becomes
$$
f_{Pool}(x) \equiv \E[Y_S|X_S=x] = \E[g(S,X_S)|X_S=x].
$$
This rule satisfies equal treatment --- it has no functional dependence on $S$
--- but we argued earlier (in Section~\ref{s:cond}) that this forecast implicitly
uses the information in the portfolio features $x$ to infer bank identity.

To introduce a nonlinear version of the FEO forecast, we will make the relatively
modest assumption that (\ref{ynon}) admits a decomposition of the form
\begin{equation}
Y_S = f_0 + f_1(S) + f_2(X_S) + \epsilon, \quad \E[\epsilon|S]=\E[\epsilon|X_S]=0,
\label{ygam}
\end{equation}
with $f_0 = \E[Y_S]$, $f_1:\{1,\dots,\bar{S}\}\to\R$, $f_2:\R^d\to\R$, and
\begin{eqnarray}
\E[Y_S - f_0 - f_1(S)|X_S] &=& f_2(X_S)
\label{res1} \\
\E[Y_S - f_0 - f_2(X_S)|S] &=& f_1(S),
\label{res2}
\end{eqnarray}
$\E[f^2_1(S)]<\infty$, $\E[f^2_2(X_S)]<\infty$, and
$$
\E[f_1(S)] = \E[f_2(X_S)]=0.
$$

Equations (\ref{res1})--(\ref{res2}) are population versions of the backfitting
algorithm in Hastie and Tibshirani \cite{hastie}, which is a special case of the
alternating conditional expectations algorithm of Breiman and Friedman \cite{ace}.
Given an initial choice of $f_1$ (and known $f_0$), (\ref{res1}) defines an initial
choice of $f_2$ through the regression of the residual $Y_S - f_0 - f_1(S)$ on
$X_S$. Equation (\ref{res2}) then defines an updated choice of $f_1$. The
algorithm iterates over (\ref{res1}) and (\ref{res2}). In writing (\ref{ygam}), we
are positing that this algorithm has a fixed point. Convergence of the backfitting
algorithm is established under widely applicable conditions in Ansley and Kohn \cite{anskoh}.

We now introduce
\begin{equation}
\hat{Y}_F(x) = f_0 + f_2(x)
\label{yfnon}
\end{equation}
as a nonlinear counterpart of the FEO forecast. We justify this interpretation
by showing that $\hat{Y}_F$ exhibits properties that are nonlinear counterparts
of the key properties of the FEO forecast in Section~\ref{s:feo} and~\ref{s:unified}.
To state the result, consider forecasts of the form
\begin{equation}
\hat{Y}_{\gamma}(x) = f_0 + f_2(x) + \gamma(x),
\label{bgnon}
\end{equation}
for some $\gamma:\mathbb{R}^d\to\mathbb{R}$ with $\E[\gamma(X_S)^2]<\infty$.

\begin{proposition}\label{p:non}
The nonlinear FEO forecast (\ref{yfnon}) satisfies
\begin{equation}
\cov[\hat{Y}_F(X_S)-Y_S,X_S - \E[X_S|S]]=0.
\label{covnon1}
\end{equation}
For $\hat{Y}_{\gamma}$ as in (\ref{bgnon}), if $\gamma$ reduces errors,
in the sense that $\E[(\hat{Y}_{\gamma}(X_S)-Y_S)^2]
<  \E[(\hat{Y}_F(X_S)-Y_S)^2]$, then it misdirects legitimate information,
in the sense that
\begin{equation}
\cov[\gamma(X_S), \E[f_1(S)|X_S]]>0
\label{covnon2}
\end{equation}
and
\begin{equation}
\cov[\E[\gamma(X_S)|S], f_1(S)]>0.
\label{covnon3}
\end{equation}
\end{proposition}

Property (\ref{covnon1}) parallels the condition in row (iv) of
Table~\ref{t:uni} that characterizes the FEO forecast in the linear setting. 
It says that the forecast error $\hat{Y}_F(X_S)-Y_S$ is uncorrelated
with the legitimate information $X_S - \E[X_S|S]$, which is the component
of $X_S$ orthogonal to bank identity $S$.
Properties (\ref{covnon2})--(\ref{covnon3}) parallel conditions (i) and (ii)
in Proposition~\ref{p:nmli}.
In particular, in (\ref{covnon2}), $\E[f_1(S)|X_S]$ is the 
expected impact of bank identity inferred from portfolio features; 
the positive covariance with $\gamma(X_S)$
thus indicates that $\gamma$ misdirects some of the information in $X_S$
to inferring $S$.

We briefly contrast our FEO forecast in (\ref{yfnon}) with an alternative approach
to extending fairness concerns to complex, nonlinear models. The alternative
seeks to strip $X_S$ of any protected attributes before a model is estimated.
Examples of this general approach include Gr{\H u}new{\H a}lder and Khaleghi \cite{grune}
and Madras et al.\ \cite{madras}. This approach is primarily concerned with ensuring 
demographic parity: if a model has no access --- not even indirect access ---
to a protected attribute, its forecasts will be independent of the attribute. But we
argued previously that demographic parity is too strong a condition for our setting.
Our FEO forecast in (\ref{yfnon}) treats all the information in $X_S$ as legitimate
information --- even elements that could help infer $S$ --- but it ensures that the
information is not in fact misdirected to infer $S$.

\section{Empirical Evidence}\label{sec:experiment}

In this section, we document empirical evidence of heterogeneity in bank-specific
models of loss rates, and we examine the implications of this heterogeneity for
the choice of an industry-wide model. We find strong evidence of statistically
significant differences in model parameters across banks. These differences can
lead to material differences between pooled and FEO coefficients in an industry model.

We must emphasize, however, that our investigation is constrained by the very
limited information made publicly available by banks about the risk characteristics
and losses in their loan portfolios. The Federal Reserve has far more granular
information about banks' loans and losses. Our results can therefore provide only a
rough indication of the impact of bank heterogeneity in the Fed's stress tests.

\subsection{Data}

We use two types of data and data sources: historical
macroeconomic data and loan information for individual banks. 

\subsubsection{Macroeconomic Data}
\label{s:macro}

We use data on seven of the macro variables used in the Federal Reserve's stress tests:
real disposable income growth, real GDP growth, house price index level, inflation rate,
unemployment rate, Dow Jones total stock index level, and the Treasury spread. The
Federal Reserve provides historical data on its website for all variables used in forming
stress scenarios, including these. We use the values reported by the Fed for these
variables in the June 2020 stress test; these values run from 1990 through 2019.

We aggregate these variables into a single macro variable by taking the first principal
component of their correlation matrix. Table~\ref{tbl:pca} shows the corresponding
loadings. We see that an increase in the principal component corresponds to decreases
in income growth and GDP growth and an increase in unemployment, suggesting that
this composite variable serves as a reasonable measure of overall economic conditions.
Figure~\ref{fig:pca} plots the level of this variable over time and shows a sharp climb
around 2008 and 2020.%
\footnote{The loadings in Table~\ref{tbl:pca} are calculated using data through 2019,
and we use these loadings to extend PC1 through the end of 2021. When we include
the COVID period in the calculation of the principal components, PC1 becomes
harder to interpret. For example, the coefficients for income growth and
unemployment have the same sign.}

\begin{table}[H]
\centering
\begin{tabular}{lr}
Macro Factor &  PC1 Loading \\
\hline
Real disposable income growth &  -0.229 \\
Real GDP growth               &  -0.525 \\
Change House Price Index      &  -0.467 \\
CPI inflation rate            &  -0.079 \\
Change unemployment           &   0.529 \\
Change Dow                    &  -0.293 \\
Change Treasury Spread        &   0.287 \\
\hline
\end{tabular}
\caption{Loadings of first principal component on macro variables.}\label{tbl:pca}
\end{table}

\begin{figure}[H]
    \centering
    \includegraphics[width = 0.6\textwidth]{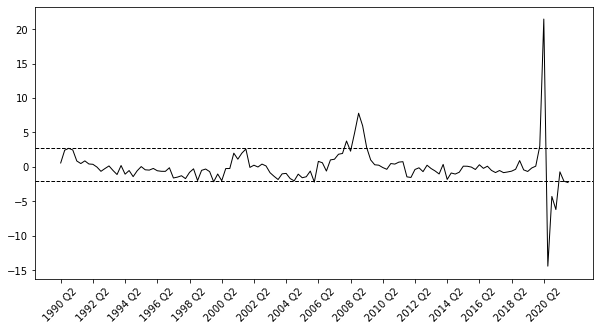}
    \caption{First principal component (PC1) of macro variables from 1990 Q2 to 2021 Q4. 
    The dashed lines correspond to the 5th and 95th percentiles of PC1.}\label{fig:pca}
\end{figure}

\subsubsection{Loan Information for Individual Banks.} 

Bank holding companies publicly report financial information quarterly
through the Federal Reserve's form Y-9C. We use these filings to
collect information on four loan types that are treated separately
in the Fed's stress tests: credit cards, first lien mortgages,
commercial real estate loans, and commercial and industrial loans.
For each category, each bank, and each quarter, we collect loan balances,
charge-offs, recoveries, and total amounts past due serving as
our proxy measures of loan portfolio risk.

We collect this data from 2001 to 2021 for the thirty-five 
largest banks by total assets (as of December 2021). 
The banks are listed in Table~\ref{t:names}.
The stress test focuses on adverse economic conditions;
we weight each observation by
the level of stress in each quarter and each bank's load balance:
for each bank $s$, each quarter $t$, and loan category $p$, we weight the
observations by
\begin{equation}\label{eqn:weight}
    w_{s,t}^p = e^{\lambda \textit{MacroPC}_t} \times \textit{Loan}_{s,t}^p,
\end{equation}
where $\textit{Loan}_{s,t}^p$ is the size of the loan portfolio of type $p$
for bank $s$ in quarter $t$.
We choose $\lambda$ so that for the same loan level, 
the worst economic quarter (as measured by $\textit{MacroPC}_t$)
is given twice the weight as the best quarter.

We would prefer to conduct our analysis using data from stress periods only,
but that would leave us with too few observations. 
Weighting by the level of stress in a quarter allows us to approximate the effect
of conditioning on stress while making greater use of the available data.
This approach relies on the assumption that data from non-stressful periods
is relevant to forecasting losses in periods of stress.

We merger-adjust all bank data. For example, Truist Financial, one of the banks
in Table~\ref{t:names}, was formed from the 2019 merger of BB\&T and SunTrust,
so our data for Truist in earlier years combines data from those two banks. We
repeat this process as we work backwards in time. We obtain information on
mergers and acquisitions from the Federal Financial Institutions Examination
Council website. (We have also run our analysis without merger-adjusting the data;
doing so does not change our conclusions and generally increases heterogeneity
across banks.)

In each loan category, we calculate a loss rate (net charge-off rate) for each bank
$s$ and each quarter $t$ as the ratio
\begin{equation}\label{eqn:normal}
    \textit{LossRate}_{s,t} =  \frac{\textit{Charge-offs}_{s,t} - 
    \textit{Recoveries}_{s,t}}{\textit{Total Loans in Category}_{s,t-1}}.
\end{equation}

This measure is commonly used in stress testing; see, for example, 
Guerrieri and Welch \cite{guewel}, Hirtle et al.\ \cite{class}, 
and Kapinos and Mitnik \cite{kapmit}.
We similarly normalize the amounts past due to get a 
$\textit{PastDueRate}_{s,t}$ for each bank-quarter.
We remove values less than $-50$\% or greater than $50$\% of 
\textit{LossRate} and values greater than $20$\% of \textit{PastDueRate}.
We winsorize \textit{PastDueRate} at the upper and lower 5\% levels.
To attain a mostly balanced panel for more reliable estimates, 
in each loan category we include only banks with at least 18 years (72 quarters)
of history from 2001 Q1 to 2021 Q4.

Table~\ref{tbl:descriptive} shows descriptive statistics for these variables. 
Loss rates and past due rates are shown by loan category --- credit cards (CC),
first liens (FL), commercial real estate (CRE), and commercial and industrial (CI).
Columns 2--4 of the table summarize time-averaged values across banks.
Columns 5--8 summarize observations across all banks and quarters.

\begin{table}[H]
    \centering
\begin{tabular}{l|rrr|rrrr}
\toprule
{} & \multicolumn{3}{|c|}{bank averages} & \multicolumn{4}{|c}{all observations} \\
{} &  min &  mean &  max &  lower 5\% &  mean &  upper 5\% &   std \\
\midrule
Loss Rate: CC      &      -0.20 &            2.50 &       3.36 &      0.31 &  3.00 &      5.88 &  2.01 \\
Loss Rate: FL      &       0.01 &            0.22 &       0.66 &     -0.01 &  0.27 &      1.03 &  0.51 \\
Loss Rate: CRE     &       0.06 &            0.19 &       0.46 &     -0.04 &  0.16 &      1.00 &  0.41 \\
Loss Rate: CI      &       0.05 &            0.41 &       1.00 &      0.00 &  0.42 &      1.56 &  0.52 \\ \hline
Past Due Rate: CC  &       1.10 &            2.90 &       4.23 &      1.06 &  3.30 &      5.74 &  1.53 \\
Past Due Rate: FL  &       0.71 &            4.01 &       8.02 &      0.45 &  6.20 &     11.80 &  4.60 \\
Past Due Rate: CRE &       1.07 &            2.15 &       3.80 &      0.36 &  2.12 &      6.52 &  1.88 \\
Past Due Rate: CI  &       0.05 &            1.65 &       2.85 &      0.05 &  1.74 &      3.97 &  1.23 \\
\bottomrule
\end{tabular}
    \caption{Descriptive statistics in percent. Columns 2--4 are calculated from banks' time averages, and columns 5--8 are calculated from all observations, with mean and standard deviation are stressed time and loan balance weighted.}
    \label{tbl:descriptive}
\end{table}

Figure \ref{fig:past_due_het_w} plots the mean past due rate
($\pm1.96$ standard errors) for each bank in each loan category.
The banks are identified by their stock tickers. 
The figure illustrates substantial heterogeneity across banks in their
loan portfolios. 
For example, Bank of America (BAC) has among the highest past due rates
for credit card loans,
but in the commercial real estate category it has among the lowest.
This type of pattern is consistent with the idea that
banks have different areas of specialization and may target different markets.

The widths of the bars in Figure \ref{fig:past_due_het_w} 
show differences across loan categories
and banks in the volatility of their past due rates. 
We again observe significant heterogeneity among different banks.

\begin{figure}[H]
    \centering
    \includegraphics[width = 1.0\textwidth]{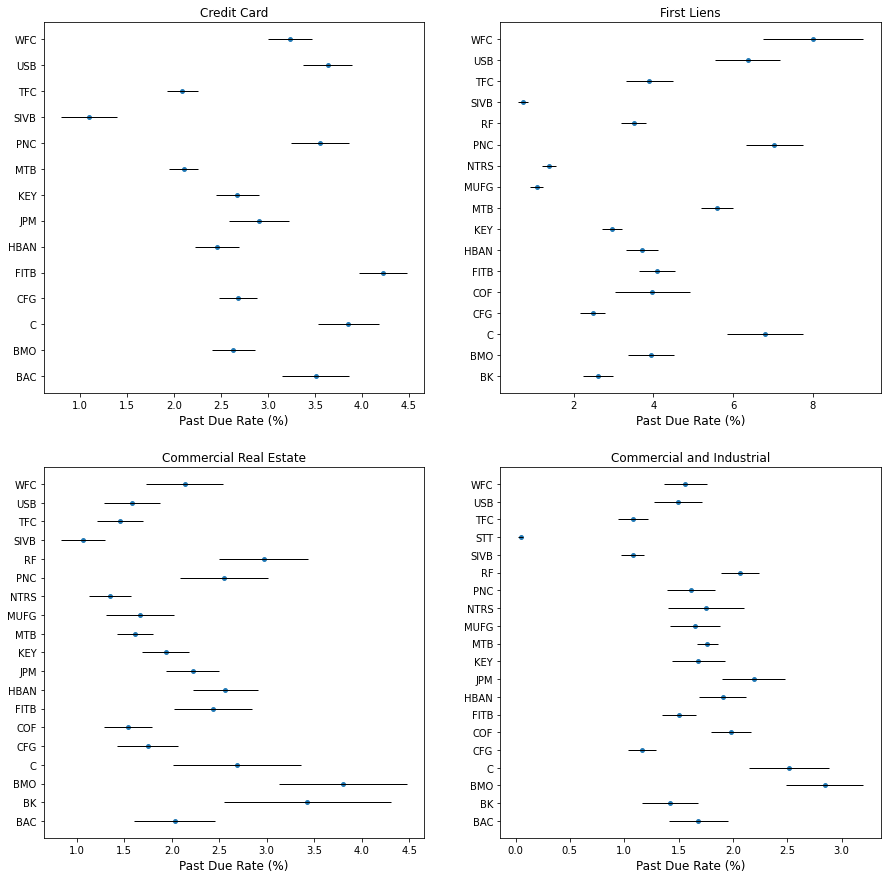}
    \caption{Past due rates (winsorized) by bank and loan category. The dots show mean values and each horizontal bar corresponds to $\pm 1.96$ standard errors.} 
    \label{fig:past_due_het_w}
\end{figure}

\subsection{Heterogeneity in Slopes and Intercepts}
\label{s:heterogeneity}

We use the bank data to approximate our
theoretical framework through the specification
\begin{equation}
\textit{LossRate}_{s,t} = \alpha_s + \beta_s \textit{PastDueRate}_{s,t-l} 
+ \gamma_s \textit{MacroPC}_{t-l}
+ \epsilon_{s,t},
\label{pdreg}
\end{equation}
for bank $s$ in quarter $t$, where \textit{MacroPC} is the principal component of the
macro variables introduced in Section~\ref{s:macro}.
(In Section~\ref{appendix:robust-precovid}, we also include allowances in (\ref{pdreg})
as a robustness check.)
The lag $l$ is four quarters to mimic the stress testing's forward-looking framework. 
We estimate separate coefficients for each of the four loan categories, for each bank,
and the observations are loan balance and stress weighted using (\ref{eqn:weight}). 
Because these are bank-specific regressions, we do not add bank-specific controls.

For each loan category, we want to test for heterogeneity in parameters across banks.
When we test for heterogeneity, the null hypothesis states that slopes for all banks are equal,
\begin{equation}\label{test:beta}
    H_0: \beta_1 = \cdots = \beta_{\bar{S}},
\end{equation}
or that the intercepts are equal,
\begin{equation}\label{test:alpha_beta}
    H_0: \alpha_1 = \cdots = \alpha_{\bar{S}}.
\end{equation}
The alternative hypothesis in each case states that 
the indicated parameters are not identical across banks.
We will run these tests with different subsets of the variables in (\ref{pdreg}) included
and interpret the coefficients in (\ref{test:beta}) accordingly.

To test these hypotheses for a particular loan category, 
let $X_s$ be the $n_s$ by $k$ data matrix for bank $s$,
where $n_s$ is the number of observations for bank $s$ in the loan category,
and $k=$1 or 2 is the number of variables included on the right side of (\ref{pdreg}).
Let $\tilde{X}_s = (1, X_s)$ be $X_s$ concatenated with a column of 1s,
and let $X^*$ be the diagonal block matrix 
$X^* = \text{diag}(\tilde{X}_1, .., \tilde{X}_S)$. 
Let $\theta^* = (\alpha_1, \beta_1^\top, ..., \alpha_S, \beta_S^\top)^\top, 
\epsilon^*=(\epsilon_1, ..., \epsilon_S)$, 
where $\epsilon_s$ is a column vector of length $n_s$. 
Our unrestricted model can be written as 
\begin{equation}\label{eqn:SUR_rep}
    Y = X^* \theta^* + \epsilon^*,
\end{equation}
and the restrictions in (\ref{test:beta}) and (\ref{test:alpha_beta}) impose
linear constraints on the parameter $\theta^*$.

We apply the Wald test to test linear constraints on $\theta^*$ 
in (\ref{eqn:SUR_rep}) under various assumptions
on the error covariance matrix. We consider (i) bank-clustered errors, 
which allows correlation in errors across quarters for each bank, but no correlation across banks; 
(ii) time-clustered errors, 
which allows correlation in errors across banks in each quarter, but no correlation across time.

Table~\ref{tab:full_het_1y} reports $p$-values for the tests when
different subsets of variables are included on the right side of (\ref{pdreg}),
for a forecast horizon of one year.
All tests indicate strong evidence of heterogeneity in the
intercepts, the coefficients for past due rates, and macro variables.

Next we examine the impact of heterogeneity. 
Table~\ref{tbl:1y_reg} compares pooled and FEO coefficients for \textit{PastDueRate} 
using a one-year lag when \textit{MacroPC} is and is not included in (\ref{pdreg}).
We estimate $\beta_{Pool}$ in a pooled panel regression, 
and $\beta_F$ in a panel regression with bank fixed effects included. 
Both regressions are weighted using (\ref{eqn:weight}).

In Table~\ref{tbl:1y_reg}, the columns labeled ``diff'' show the difference
in estimates $\beta_F - \beta_{Pool}$, serving as 
a measure of the impact of addressing heterogeneity
in choosing an industry model.
The table also shows $p$-values for tests of
$H_0: \beta_{Pool} = \beta_F$ (or $\gamma_{Pool} = \gamma_F$ in the case
of the macro variable). 
To calculate these $p$-values, we estimate the pooled
and FEO models simultaneously, as follows.
Let $\tilde{X}_{Pool} = (1, X)$ be $X$ concatenated with a column of 1s,
and let $\tilde{X}_F = (U, X) $ be $X$ concatenated with columns
corresponding to centered bank identity variables $U$.
Let $X_*$ be the diagonal block matrix $X_* = \text{diag}(\tilde{X}_{Pool}, \tilde{X}_F)$
and $\theta_* = (\alpha_{Pool}, \beta_{Pool}^\top, \delta_1, \delta_2, \dots, \delta_{\bar{S}}, \beta_F^\top)^\top$. 
Then we have 
\[
Y = X_* \theta_* + \epsilon_*,
\]
and testing $H_0$ is equivalent to testing linear constraints 
on the parameters $\theta_*$,
for which we apply the Wald test.
The macro variable captures common variability over time,
so we cluster errors by bank.

The results in Table~\ref{tbl:1y_reg} show that 
the differences between the pooled and FEO estimates
are significant in three of the four loan categories.
Moreover, the differences can be material.
For example, for first lien loans in the top panel of Table~\ref{tbl:1y_reg},
an absolute difference of 0.015 translates to a relative difference of $24\%$ 
($=|\beta_{Pool} - \beta_{F}|/\beta_{F}$),
which can have a large relative impact on predicted losses.
From Table~\ref{tbl:descriptive}, we see that 
the average bank has a past due rate of 4.01\%
on FL loans. The difference $0.015 \times 4.01\% = 0.060\%$ 
is 27\% of the average FL loss rate of 0.22\% in Table~\ref{tbl:descriptive}.
The additional capital required to offset the higher predicted loss rate would
be 27\% of the capital required to offset the average loss rate.

To further analyze the differences in forecasts,
we consider the relative prediction differences given by 
$|(\hat{Y}_{Pool}(x) - \hat{Y}_{F}(x))/Y(x)|$, where the denominator
is the observed loss rate.
Table~\ref{tbl:pred_diff} reports the mean and median of these 
relative differences
in each of the four loan categories, for all banks in all quarters. 
In each loan category, the mean relative difference is large; moreover, in each
case the mean is appreciably larger than the median, reflecting the presence
of some very large relative prediction differences, which could be particularly important.
To illustrate the differences,
Figure~\ref{fig:fitted} plots histograms of 
the FEO and pooled fitted values for Citigroup's first lien loans.
The comparison shows, in particular, that the frequency of the largest and smallest
predicted loss rates differ between the two methods.

Our main results use data through 2021. As a robustness check, we
run our analysis using data through 2019. This truncation serves two purposes.
It ensures that our conclusions are not driven by a few extreme values
during the COVID period 2020--2021, and it accounts for a change in
how banks measure allowances (the Current Expected Credit Losses
methodology) beginning at the end of 2019. 
We also consider including banks' allowances for losses as another proxy 
for the portfolio risk.
Because allowances are not consistently reported separately by loan
category, we use banks' total allowances across all loan types.
That is, for each bank-quarter we calculate 
$\textit{AllowanceRate}_{s,t}$ using (\ref{eqn:normal}),
but normalizing by the total loans in all categories.
We repeat our tests with pre-COVID data and the 
addition of allowance rates. 
The results, reported in Section~\ref{appendix:robust-precovid}, 
are similar to those reported in this section.

Our results document evidence of bank heterogeneity and
its potential impact on loss forecasts.
We have not sought to identify the drivers of heterogeneity;
that would require a very different investigation,
particularly since some of the most interesting potential drivers
(a bank's management, the quality of its IT systems) are
difficult to measure. Guerrieri and Harkrader \cite{guehar},
for example, find that bank-specific factors account for a sizable
fraction of the variation in bank performance.
But they measure the bank-specific component as the residual
in a regression that removes the effect of macroeconomic and banking-wide factors;
they do not identify specific bank features that influence performance.
Some examples of bank features used as controls in stress testing models
can be found in Hirtle et al.\ \cite{class},
Kapinos and Mittnic \cite{kapmit}, and Kupiec \cite{kupiec}.
These are balance sheet features, and they are usually found to be
more relevant to forecasting revenues than losses.
We discuss revenue models in Section~\ref{appendix:ppnr}.

\begin{table}[]
    \centering
    \begin{tabular}{l|rrrrrrrrrrrr}
\toprule
Covariance & \multicolumn{4}{|c|}{$\alpha$} & 
\multicolumn{4}{|c|}{$\beta_{PDR}$} & \multicolumn{4}{|c}{$\gamma$}\\
Estimation &    \multicolumn{1}{|c}{CC} &   FL &  CRE &     CI &     \multicolumn{1}{|c}{CC} &   FL &  CRE &     CI &      \multicolumn{1}{|c}{CC} &     FL &    CRE &     CI \\
\midrule
bank clustered &  0.00 &  0.00 &  0.00 &  0.00 &  0.00 &  0.00 &  0.00 &  0.00 &       &       &       &       \\
time clustered &  0.00 &  0.00 &  0.00 &  0.00 &  0.00 &  0.00 &  0.00 &  0.00 &       &       &       &       \\ \hline
bank clustered &  0.00 &  0.00 &  0.00 &  0.00 &  0.00 &  0.00 &  0.00 &  0.00 &  0.00 &  0.00 &  0.00 &  0.00 \\
time clustered &  0.00 &  0.00 &  0.00 &  0.00 &  0.00 &  0.00 &  0.00 &  0.00 &  0.00 &  0.00 &  0.00 &  0.00 \\
\bottomrule
\end{tabular}
    \caption{$P$-values for heterogeneity tests. In each loan category, 
    the first two rows are for a model with \textit{PastDueRate} only, and the
    last two rows are for a model with \textit{PastDueRate} and \textit{MacroPC}.
    The two rows for each model show
    results under alternative assumptions on the error covariance matrix.}
    \label{tab:full_het_1y}
\end{table}

\begin{table}[]
\centering
\begin{tabular}{l|rrrlrrrl}
\toprule
Loan & \multicolumn{4}{|c|}{\textit{Past Due Rate}} & \multicolumn{4}{|c}{\textit{Macro PC}} \\
Type &  $\beta_{Pool}$ &  $\beta_F$ & diff &   \multicolumn{1}{c|}{$p$-value} & 
$\gamma_{Pool}$ & $\gamma_{F}$ & diff &             $p$-value \\
\midrule
 CC & 0.782 & 0.833 & -0.050 & 0.009$^{***}$ &       &       &        &       \\
 FL & 0.047 & 0.062 & -0.015 & 0.001$^{***}$ &       &       &        &       \\
CRE & 0.131 & 0.129 &  0.002 & 0.464         &       &       &        &       \\
 CI & 0.208 & 0.219 & -0.011 & 0.040$^{**}$  &       &       &        &       \\ \hline
 CC & 0.774 & 0.823 & -0.049 & 0.009$^{***}$ & 0.040 & 0.037 &  0.003 & 0.000$^{***}$ \\
 FL & 0.046 & 0.061 & -0.015 & 0.001$^{***}$ & 0.018 & 0.019 & -0.001 & 0.176 \\
CRE & 0.131 & 0.129 &  0.002 & 0.485 & 0.011 & 0.011 &  0.000 & 0.813 \\
 CI & 0.205 & 0.216 & -0.010 & 0.044$^{**}$  & 0.022 & 0.022 &  0.000 & 0.127 \\
\bottomrule 
\multicolumn{9}{r}{$^{*}$p$<$0.1; $^{**}$p$<$0.05; $^{***}$p$<$0.01} \\
\end{tabular}
\caption{Comparison of coefficients for one-year forecasts. We regress \textit{LossRate} on 
(i) \textit{PastDueRate} and (ii) \textit{PastDueRate} and \textit{MacroPC}.
Difference is calculated as $\beta_{Pool}-\beta_{F}$. $p$-values 
test
$H_0: \beta_{Pool} = \beta_F$ for \textit{PastDueRate} or 
$H_0: \gamma_{Pool} = \gamma_F$ for \textit{MacroPC}.}
\label{tbl:1y_reg}
\end{table}

\begin{table}[]
\centering
\begin{tabular}{lrrrr}
\toprule
  &  CC &  FL &  CRE &  CI \\
\midrule
     mean &      6.1 &        510.1 &                    52.0 &                       24.0 \\
   median &      2.4 &         89.5 &                     4.2 &                        3.7 \\
\bottomrule
\end{tabular}
\caption{Mean and median of relative prediction differences between the pooled and the FEO estimates (in \%).}
\label{tbl:pred_diff}
\end{table}

\begin{figure}
    \centering
    \includegraphics[width=0.6\linewidth]{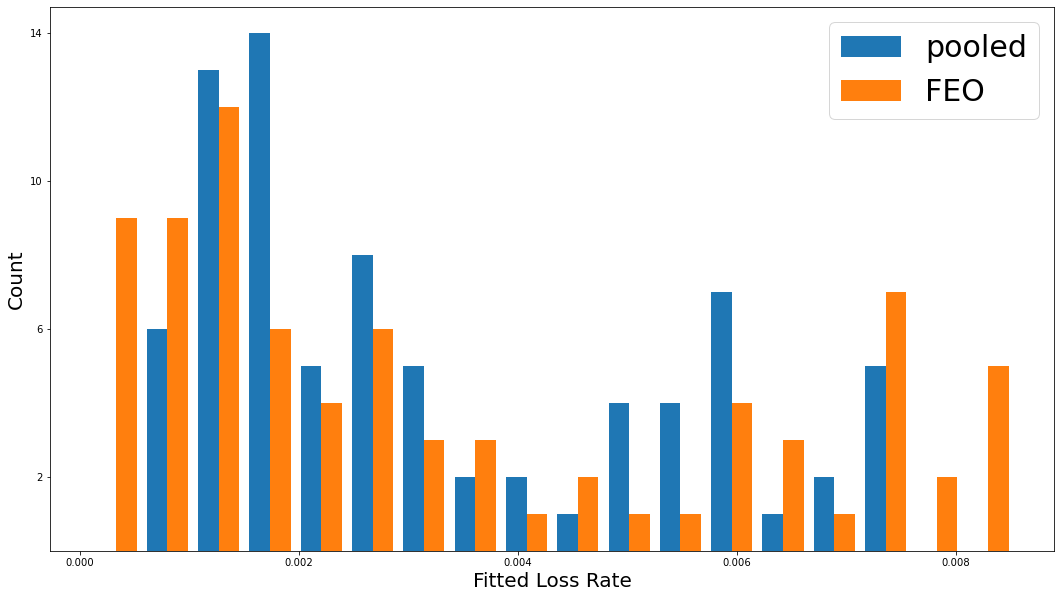}
    \caption{Pooled and FEO predicted loss rates for Citigroup's first lien loans.}
    \label{fig:fitted}
\end{figure}

\subsection{Nonlinear Models}
\label{s:robust-macro}

Building on Section~\ref{s:nonlinear},
we consider generalized additive models (GAMs) in which the effect of the 
past due rates and the macro variable are not restricted to be linear. 
More specifically, we consider the specifications
\begin{equation}\label{eqn:gam-P}
    Y_{s,t}^{Pool} = f^P_0 + f^P_1(PDR_{s,t-l}) + f^P_2(MacroPC_{t-l}) + f^P_3(PDR_{s,t-l} \times MacroPC_{t-l}) + \epsilon^P_{s,t}
\end{equation}
and 
\begin{equation}\label{eqn:gam-F}
    Y_{s,t}^F = f^F_0 + f^F_1(PDR_{s,t-l}) + f^F_2(MacroPC_{t-l}) + f^F_3(PDR_{s,t-l} \times MacroPC_{t-l}) + f^F_4(s) + \epsilon^F_{s,t},
\end{equation}
in which $f^{P/F}_i$, $i=1,2,3,$ are (possibly nonlinear) centered functions of
\textit{PastDueRate}, \textit{MacroPC}, and their interaction 
\textit{PastDueRate$\times$MacroPC}, respectively, and
$f^F_4$ measures centered bank fixed effects.
The pooled model (\ref{eqn:gam-P}) omits $f^P_4$; 
the FEO model (\ref{eqn:gam-F}) estimates $f^F_4$
but discards it in forecasting loss rates to satisfy equal treatment.
We consider the modeling of loss rates four quarters ahead, so $l=4$ in both cases.

We use the R package \texttt{gam} (Hastie \cite{hastie-r})
to fit (\ref{eqn:gam-P}) and (\ref{eqn:gam-F}),
taking the $f^{P/F}_i$, $i=1,2,3$, to be smoothing splines with 4 degrees of freedom.
We choose \texttt{gam} because it is a direct implementation
of the backfitting algorithm in Hastie and Tibshirani \cite{hastie},
which underpins the framework in Section~\ref{s:nonlinear}.
In particular, Proposition~\ref{p:non} applies to
FEO forecasts based on (\ref{eqn:gam-F}).

Model (\ref{eqn:gam-P}) is nested within model (\ref{eqn:gam-F}), so
we can use an $F$-test to compare the two.
In all four loan categories, the test rejects (with $p$-values smaller than 0.01)
the restriction to equal bank fixed effects ($f^P_4\equiv 0$) imposed in
the pooled model (\ref{eqn:gam-P}).
Figure~\ref{fig:gam-fe} plots the centered bank fixed effects $f^F_4$ 
for all four loan categories, expressed in percent.
We observe significant variability within each loan type. For example,
for credit card loans, JPM's fixed effect
is three percentage points larger than WFC's.
We also observe variability across loan types for individual banks.
For example, CFG has the highest fixed effect for first lien loans, 
but one of the lowest for credit card loans. 
These observations again reflect the notable heterogeneity among
the bank holding companies. 

Table~\ref{tbl:gam_pred_diff} reports the mean and median of 
the relative prediction differences between the FEO and the pooled predictions,
given, as in Section~\ref{sec:experiment}, 
by $|(\hat{Y}_{Pool}(x) - \hat{Y}_F(x))/Y(x)|$).
These summary statistics show that the relative prediction
differences can indeed be very large.
To further illustrate this point, Figure~\ref{fig:fitted-gam}
contrasts the prediction distributions 
for Citigroup's first lien loans using the pooled and FEO methods.
As in Proposition~\ref{p:non}, the pooled method may yield smaller prediction
errors overall, but it does so by implicitly misdirecting legitimate information.

\begin{figure}[H]
    \centering
    \includegraphics[width = 1.0\textwidth]{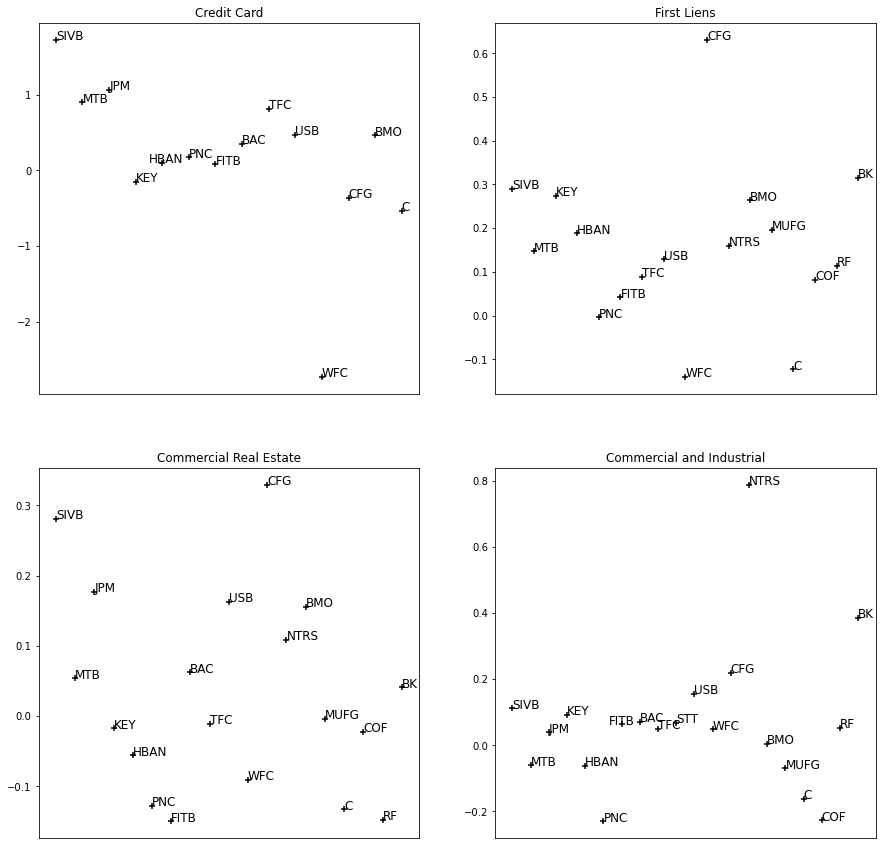}
    \caption{Banks' generalized fixed effects. Y-axis is in \%.} 
    \label{fig:gam-fe}
\end{figure}

\begin{table}[]
\centering
\begin{tabular}{lrrrr}
\toprule
  &  CC &  FL &  CRE &  CI \\
\midrule
mean   &         25.5 &        559.0 &                   106.6 &                       88.5 \\
median &         10.4 &        120.8 &                    14.2 &                        9.8 \\
\bottomrule
\end{tabular}
\caption{Mean and median of relative prediction differences between the pooled and the FEO estimates (in \%) for GAMs.}
\label{tbl:gam_pred_diff}
\end{table}

\begin{figure}
    \centering
    \includegraphics[width=0.6\linewidth]{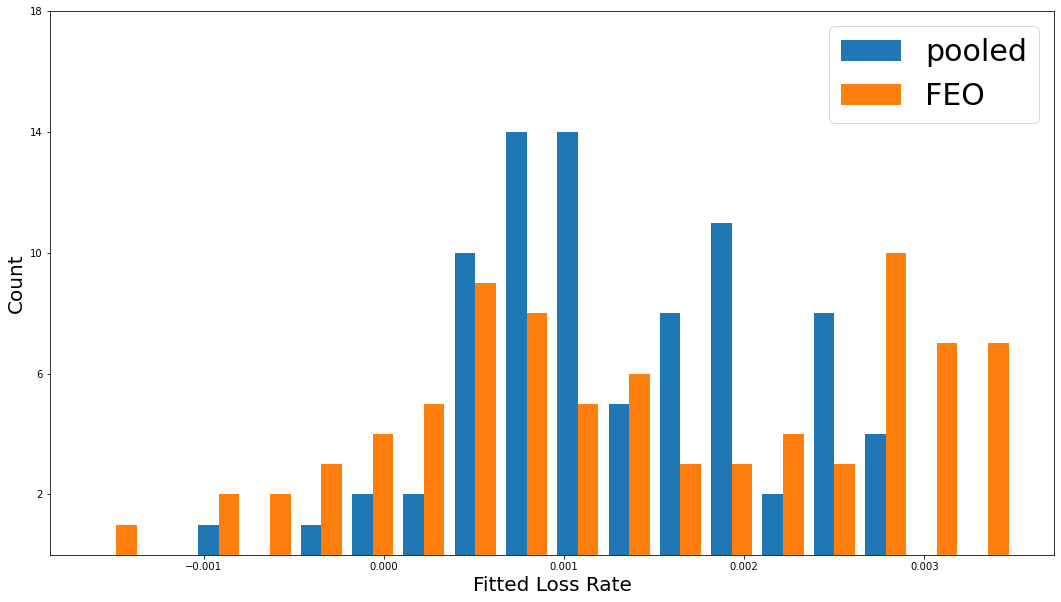}
    \caption{Pooled and FEO predicted loss rates for Citigroup's first lien loans.}
    \label{fig:fitted-gam}
\end{figure}

\section{Concluding Remarks}

The current practice of regulatory stress testing ignores bank heterogeneity
in loss models as a matter of policy and principle. We have argued that
simply pooling banks can distort coefficients on legitimate features and is
vulnerable to implicit misdirection of legitimate information to infer bank
identity. We have examined various ways of incorporating fairness
considerations and shown that estimating and discarding centered bank
fixed effects addresses the deficiencies of pooling --- and it does so in an
optimal sense.

Beyond this specific recommendation, the broader conclusion to be drawn
from our analysis is that accuracy and equal treatment can more effectively
be addressed by accounting for bank heterogeneity rather than ignoring it.
Although we have focused on the stress testing application, our analysis
applies more generally to settings requiring the fair aggregation
of individually tailored models into a single common model.
\bigskip

{\noindent\bf Acknowledgments.} For many helpful comments,
we thank the referees and our discussant (Matthew Plosser) at the
2022 Federal Reserve Stress Testing Conference.



\appendix

\section{Proofs}

\begin{proof}[Proposition~\ref{p:pool}]
Problem (\ref{sqloss}) is solved by the linear projection of $Y_S$
onto the span of 1 and $X_S$. 
If $\var[X_S]$ is invertible, then the coefficients
of the linear projection are given by (\ref{apool}) and
$$
\beta_{Pool} = \var[X_S]^{-1}\cov[Y_S,X_S];
$$
see, for example, Wooldridge \cite{wool}, p.25. 
In (\ref{vars})--(\ref{wdef}) we can write
$$
\var[X_S] = \sum_sp_sW_s = \sum_sp_s\Sigma_s + \var[\mu_S].
$$
This matrix is positive definite because we assumed that each $\Sigma_s$ is positive
definite, so $\var[X_S]=\E[W_S]$ is indeed invertible.
To evaluate $\cov[Y_S,X_S]$ for $Y_S$ in (\ref{ymix}), we first note that
$$
\cov[X_S,\epsilon_S] = \E[\cov[X_S,\epsilon_S|S]] + \cov[\E[X_S|S],\E[\epsilon_S|S]]
= \E[0] + \cov[\mu_S,0] = 0.
$$
It follows that 
\begin{eqnarray*}
\lefteqn{\cov[Y_S,X_S] } && \\
&=& \E[\cov[\alpha_S,X_S|S]] + \cov[\E[\alpha_S|S],\E[X_S|S]]
+ \E[\cov[\beta^{\top}_SX_S,X_S|S]] + \cov[\E[\beta^{\top}_SX_S|S],\E[X_S|S]] \\
&=& 0 + \cov[\alpha_S,\mu_S]
+ \E[\Sigma_S\beta_S] + \E[\var[\mu_S]\beta_S] \\
&=& \cov[\alpha_S,\mu_S] + \E[W_S\beta_S].
\end{eqnarray*}
\end{proof}

\begin{proof}[Proposition~\ref{p:ptf}]
By Proposition 3.4 of  Chzhen et al.\ \cite{chzhen} or Theorem 6 of Le Gouic et al.\ \cite{legouic},
the expected squared error is minimized subject to demographic parity by the rule
that assigns to bank $s$ with features $x$ the loss forecast
\begin{equation}
\hat{Y}_{PTF}(x,s) = \sum_i p_i F^{-1}_i(F_s(\alpha_s+\beta_s^{\top}x)),
\label{yptf}
\end{equation}
where $F_s$ is the cumulative distribution function of $\alpha_s + \beta_s^{\top}X_s$.
By construction, $F_s$ is then also the cumulative distribution function of 
$\alpha_s^o + \beta_s^{o\top}Z_s$, which is normal with mean $\alpha^o_s$
and variance $\|\beta_s^o\|^2$. Writing $\Phi$ for the standard normal distribution function,
we get
$$
F_s(y) = \Phi\left(\frac{y-\alpha_s^o}{\|\beta^o_s\|}\right), \quad
F^{-1}_i(q) = \alpha_i^o + \|\beta_i^o\|\Phi^{-1}(q).
$$
Making these substitutions in (\ref{yptf}) and writing $\alpha_s^o + \beta_s^{o\top}z_s$
for $\alpha_s+\beta_s^{\top}x$, with $z_s = \Sigma^{-1}_s(x-\mu_s)$, we get
\begin{eqnarray*}
\hat{Y}_{PTF}(x,s) 
&=& \sum_i p_i F^{-1}_i(F_s(\alpha_s^o + \beta_s^{o\top}z_s)) \\
&=& \sum_i p_i F^{-1}_i(\Phi(\beta_s^{o\top}z_s/\|\beta_s^o\|)) \\
&=& \sum_i p_i \{\alpha_i^o + \|\beta_i^o\|\Phi^{-1}(\Phi(\beta_s^{o\top}z_s/\|\beta_s^o\|))\} \\
&=& \sum_i p_i \{\alpha_i^o + \|\beta_i^o\| \beta_s^{o\top}z_s/\|\beta_s^o\|\},
\end{eqnarray*}
which is (\ref{ptfn}). (Demographic parity holds because the distribution of 
$\beta_s^{o\top}Z_s/\|\beta_s^o\|$ does not depend on $s$.)
Under (\ref{bprop}),
$\|\beta_i^o\| \beta_s^o/\|\beta_s^o\| = 
\|a_i\beta\| a_s\beta/\|a_s\beta\|
= a_i\beta = \beta_i^o$,
and we get (\ref{ptfs}).
\end{proof}

\begin{proof}[Proposition~\ref{p:feo}]
We can rewrite $\hat{Y}(x,s)$ in (\ref{ytilde}) as
$$
\hat{Y}(x,s) = \sum_{i=1}^{\bar{S}} a_i\mathbf{1}\{s=i\} + \beta^{\top}x,
$$
for suitable $a_i$. Minimizing (\ref{sqlossf}) over the $a_i$ and $\beta$
yields the same value for $\beta$ as minimizing (\ref{sqlossf}) using
(\ref{ytilde}) because the indicators $\mathbf{1}\{s=i\}$ have the same span as the $U_i(s)$
and a constant. Thus, the $\beta_F$ defined by (\ref{sqlossf}) is the coefficient of $X_S$ in
the regression of $Y_S$ on $X_S$ and the indicators $\mathbf{1}\{S=i\}$.
By the Frisch-Waugh-Lovell Theorem (as in Angrist and Pischke \cite{angpis}, pp.35--36),
we can therefore evaluate $\beta_F$ as the coefficient in the regression of
$Y_S$ on the component of $X_S$ orthogonal to the other variables,
which in our case are the indicators.
The projection of $X_S$ onto the indicators is given by $\sum_i\mu_i\mathbf{1}\{S=i\} = \mu_S$, so the orthogonal component is $X_S-\mu_S$.
We may therefore evaluate $\beta_F$ as the coefficient in the regression
of $Y_S-\E[Y_S]$ on $X_S-\mu_S$, which is (\ref{bfw}).
For the first factor in (\ref{bfw}), we have
$$
\var[X_S-\mu_S] = \E[\var[X_S-\mu_S|S]] + \var[\E[X_S-\mu_S|S]] = \E[\Sigma_S] + 0.
$$
For the second factor, we similarly have
$$
\cov[X_S-\mu_S,Y_S] = \E[\cov[X_S-\mu_S,Y_S|S]] = 
\E[\cov[X_S-\mu_S,\beta_S^{\top}X_S|S]] =
\E[\Sigma_S\beta_S],
$$
so (\ref{bfeo}) follows. The optimal $\alpha_F$ in (\ref{sqlossf}) ensures that
$\E[\hat{Y}(X_S,S)] = \E[Y_S]$, which yields (\ref{afeo}).
\end{proof}

\begin{proof}[Proposition~\ref{p:bbld}]
The minimization in (\ref{sqlossf}) yields coefficients $\alpha_F$,
$\delta$, and $\beta_F$, with which we can write
\begin{equation}
Y_S = \alpha_F + \sum_{i=1}^{\bar{S}-1}\delta_iU_i(S) + \beta_F^{\top}X_S + u,
\label{ysu}
\end{equation}
where the error $u$ has mean zero and is uncorrelated with $U(S)$ and $X_S$. We thus have
\begin{eqnarray*}
\beta_{Pool}
&=& \var[X_S]^{-1}\cov[X_S,Y_S] \\
&=& \var[X_S]^{-1}\{\cov[X_S,\beta_F^{\top}X_S] + \cov[X_S,\delta^{\top}U(S)]\} \\
&=& \var[X_S]^{-1}\{\var[X_S]\beta_F + \cov[X_S,U(S)]\delta\} \\
&=& \beta_F + \Lambda\delta,
\end{eqnarray*}
using the expression for $\Lambda$ in (\ref{lamdef}) for the last step.

Next, we evaluate $\delta$. Using (\ref{ysu}), we can derive $\delta$ as
the vector of coefficients
in a regression of $Y_S - \beta_F^{\top}X_S$ on $U(S)$.
Thus,
\begin{eqnarray}
\delta &=& 
\var[U(S)]^{-1}\cov[U(S),Y_S - \beta^{\top}_FX_S] \nonumber \\
&=& \var[U(S)]^{-1}\cov[U(S),Y_S]
-\var[U(S)]^{-1}\cov[U(S),X_S]\beta_F.
\label{drep}
\end{eqnarray}
To evaluate $\var[U(S)]^{-1}$, we first note that
$$
\var[U(S)] = 
\left(\begin{array}{cccc}
p_1-p_1^2 & -p_1p_2 & \cdots & -p_1p_{\bar{S}-1} \\
-p_2p_1     & p_2-p_2^2 & \cdots & -p_2p_{\bar{S}-1} \\
\vdots        &   \vdots               & \ddots  & \vdots \\
-p_{\bar{S}-1}p_1 &-p_{\bar{S}-1}p_2 & \cdots & p_{\bar{S}-1}- p^2_{\bar{S}-1}
\end{array}
\right);
$$
direct multiplication then verifies that
$$
\var[U(S)]^{-1} = 
\left(\begin{array}{cccc}
1/p_1+1/p_{\bar{S}} & 1/p_{\bar{S}} & \cdots & 1/p_{\bar{S}} \\
1/p_{\bar{S}}     & 1/p_2 + 1/p_{\bar{S}}& \cdots & 1/p_{\bar{S}} \\
\vdots        &   \vdots               & \ddots  & \vdots \\
1/p_{\bar{S}} &1/p_{\bar{S}} & \cdots & 1/p_{\bar{S}-1} +1/p_{\bar{S}}
\end{array}
\right).
$$
The vector $\cov[U(S),Y_S]$ has elements
$$
[\cov[U(S),Y_S]]_s = p_s(\E[Y_s]-\E[Y_S]), \quad s=1,\dots,\bar{S}-1;
$$
and row $s$ of the matrix $\cov[U(S),X_S]$ is given by $p_s(\mu_s-\bar{\mu})^{\top}$.
Thus, for $s=1,\dots,\bar{S}-1$, we have the vector elements
$$
(\var[U(S)]^{-1}\cov[U(S),Y_S])_s
=
(\E[Y_s]-\E[Y_S]) + \sum_{i=1}^{\bar{S}-1}p_i(\E[Y_i]-\E[Y_S])/p_{\bar{S}}
=\E[Y_s]-\E[Y_{\bar{S}}],
$$
and similarly row $s$ of the matrix $\var[U(S)]^{-1}\cov[U(S),X_S]$ is given by
\begin{equation}
(\mu_s-\bar{\mu})^{\top} + \sum_{i=1}^{\bar{S}-1}p_i(\mu_i-\bar{\mu})^{\top}/p_{\bar{S}}
=(\mu_s - \mu_{\bar{S}})^{\top}.
\label{msol}
\end{equation}
Combining these terms in (\ref{drep}) yields (\ref{delta}).
\end{proof}

\begin{proof}[Proposition~\ref{p:seo}]
We derived an expression for the rows of $M$ in (\ref{msol}),
and (\ref{mtu}) follows from that expression. By applying expression (\ref{bfw})
for $\beta_F$ in (\ref{yseo2}), we see that $\hat{Y}_{SEO}$ is the claimed
projection.
\end{proof}

\begin{proof}[Corollary~\ref{c:ptfseo}]
From (\ref{bfeo}) we know that if $\Sigma_s\equiv \Sigma$ then
$\beta_F = \E[\beta_S]$.  From (\ref{ptfs}), we get
$\bar{\beta}^o = \sum_ip_i\beta_i^o = \sum_ip_i\Sigma^{1/2}\beta_i = \Sigma^{1/2}\E[\beta_S] = \Sigma^{1/2}\beta_F$.
Thus, $\beta_F^{\top}(x-\mu_s) = \bar{\beta}^{o\top}\Sigma^{-1/2}(x-\mu_s) = 
\bar{\beta}^{o\top}z_s$. It follows that (\ref{ptfs}) and (\ref{yseo2}) coincide because
they have the same overall mean. If the distribution of $Z_s$ does not
depend on $s$, then (\ref{ptfs}) satisfies demographic parity.
\end{proof}

\begin{proof}[Proposition~\ref{p:unified}]
The claim for (i) simply restates Proposition~\ref{p:pool}. The constraint in (ii) is
demographic parity, so the optimizer follows from the definition of projection to fairness.
The constraint in (v) requires $\cov[\lambda^{\top}U(S)+\beta^{\top}X_S,U(S)]=0$.
Rearranging this equation, we get $\lambda = -(\var[U(S)])^{-1}\cov[U(S),X_S]\beta$;
i.e., $\lambda = -M\beta$. Making this substitution in the form of $\hat{Y}_S$
in row (v), (\ref{sqloss2}) becomes
\begin{equation}
\E[(Y_S - \alpha - \lambda^{\top}U(S) - \beta^{\top}X_S)^2]=
\E[(Y_S - \alpha -\beta^{\top}[X_S-M^{\top}U(S)])^2].
\label{Mloss}
\end{equation}
Minimizing this expression over $\alpha$ and $\beta$ yields the coefficients in a linear regression
of $Y_S$ on a constant $X_S-M^{\top}U(S)$. In light of Proposition~\ref{p:seo}, the optimal $\beta$ in (\ref{Mloss}) is then the coefficient on $X_S-\mu_S$ in a regression of $Y_S$ on 
a constant $X_S-\mu_S$. It follows from (\ref{bfw}) that the optimal $\beta$ in (\ref{Mloss})
is therefore $\beta_F$.
Because $\E[U(S)]=0$, the minimizing $\alpha$ in (\ref{Mloss}) is the $\alpha_F$
defined by (\ref{sqlossf}). We have thus shown that the optimal forecast in row (v) is
$$
\hat{Y}_s = \alpha_F + \beta_F^{\top}(X_s - M^{\top}U) = 
\alpha_F + \lambda^{\top}U(s) + \beta_F^{\top}X_s.
$$
In case (iv), by applying (\ref{mtu}) we see that the constraint requires $\cov[Y_S-\beta^{\top}X_S,X_S-M^{\top}U(S)]=0$,
so $\beta = (\cov[X_S,X_S-M^{\top}U(S)])^{-1}\cov[Y_S,X_S-M^{\top}U(S)]$.
Using the fact that $X_S-M^{\top}U(S)$ is orthogonal to $U(S)$, we get
\begin{eqnarray*}
\lefteqn{
\cov[X_S,X_S-M^{\top}U(S)] } && \\
&=& \cov[X_S-M^{\top}U(S),X_S-M^{\top}U(S)] + \cov[M^{\top}U(S),X_S-M^{\top}U(S)] \\
&=& \var[X_S-M^{\top}U(S)],
\end{eqnarray*}
and therefore $\beta = (\var[X_S-M^{\top}U(S)])^{-1}\cov[Y_S,X_S-M^{\top}U(S)]$.
In other words, the optimal $\beta$ in (iv) is the coefficient in a linear regression of $Y_S$
on $X_S-M^{\top}U(S)$. As noted in the discussion of (v), this is $\beta_F$, and it follows
from $\E[U(S)]=0$ that the optimal $\alpha$ in (iv) is $\alpha_F$.
\end{proof}

\begin{proof}[Proposition~\ref{p:nmli}]
By construction, the least-squares projection of $Y_S$ onto a constant and $X_S$
is given by the pooled forecast, so
$$
Y_S = \E[Y_S] + \beta_{Pool}^{\top}(X_S-\bar{\mu}) + \epsilon_P,
$$
for some orthogonal error $\epsilon_P$ with a
variance $\sigma^2_P$ that does not depend on $\gamma$.
We therefore have
\begin{eqnarray*}
\E[(\hat{Y}_{\gamma}(X_S)-Y_S)^2] 
&=& 
\E[\{\hat{Y}_{\gamma}(X_S) - \E[Y_S] - \beta_{Pool}^{\top}(X_S-\bar{\mu})\}^2] + \sigma^2_P \\
&=& \E[\{(\gamma - \Lambda\delta)^{\top}(X_S-\bar{\mu})\}^2]+ \sigma^2_P,
\end{eqnarray*}
from which (i) follows.  

Using the linear projection of $Y_S$ onto $(1,U(S),X_S)$ in (\ref{ylp}), we can write
$$
Y_S= \E[Y_S] + \delta^{\top}U(S) + \beta^{\top}_F(X_S-\bar{\mu}) + \epsilon,
$$
for some orthogonal error $\epsilon$ with a variance $\sigma^2_{\epsilon}$
that does not depend on $\gamma$.
We therefore have
\begin{eqnarray*}
\E[(\hat{Y}_{\gamma}(X_S)-Y_S)^2] 
&=& \E[\{\gamma^{\top}(X_S-\bar{\mu}) - \delta^{\top}U(S)\}^2] 
+\sigma^2_{\epsilon} \\
&=& \E[\{\gamma^{\top}(X_S-\mu_S) + \gamma^{\top}(\mu_S-\bar{\mu}) - \delta^{\top}U(S)\}^2]
+\sigma^2_{\epsilon} \\
&=& \E[\{\gamma^{\top}(X_S-\mu_S) + (\gamma^{\top}M^{\top} - \delta^{\top})U(S)\}^2] 
+\sigma^2_{\epsilon}\\
&=& \E[\{\gamma^{\top}(X_S-\mu_S)\}^2] + \E[\{(\gamma^{\top}M^{\top} - \delta^{\top})U(S)\}^2]+\sigma^2_{\epsilon},
\end{eqnarray*}
where the third equality uses (\ref{mtu}), and the last equality
uses the orthogonality of $X_S-\mu_S$ and $U(S)$.
If this expression is smaller than the corresponding value with $\gamma=0$,
then (ii) must hold.
\end{proof}

\begin{proof}[Proposition~\ref{p:ate}]
We saw in the proof of Proposition~\ref{p:feo} that $X_S-\mu_S$ is
uncorrelated with the centered indicators $U_i(S)$. 
It is also uncorrelated with $V_S-\nu_S$
because
$$
\E[(X_S-\mu_S)(V_S-\nu_S)] = \sum_sp_s\E[(X_s-\mu_s)(V_s-\nu_s)]=0,
$$
under our assumption that $X_s$ and $V_s$ are uncorrelated. Similarly,
$$
\E[(X_S-\mu_S)U_i(S)V_S] = p_i\E[(X_i-\mu_i)V_i] - p_i\E[(X_S-\mu_S)V_S] = 0,
$$
so $X_S-\mu_S$ is uncorrelated with the interaction terms. 
Thus, $X_S-\mu_S$ is uncorrelated with all the elements of
$\mathcal{O} = \{1,U(S), V_S-\nu_S, U_1(S)V_S,\dots,U_{\bar{S}}(S)V_S\}$.

Starting from the representation of (\ref{ysz}) as
$$
Y_S = 
\sum_{i=1}^{\bar{S}}\mathbf{1}\{S=i\}\{\alpha_i +
\beta^{\top}_iX_S
+ \gamma^{\top}_iV_S + \epsilon_i\},
$$
we may write
\begin{eqnarray*}
Y_S 
&=& 
\beta^{\top}_S(X_S-\mu_S) + 
\sum_{i=1}^{\bar{S}}\mathbf{1}\{S=i\}(\alpha_i +\beta^{\top}_i\mu_i)
+\sum_{i=1}^{\bar{S}}\mathbf{1}\{S=i\}\gamma_i^{\top}V_S + \epsilon_S \\
&\equiv & 
\beta^{\top}_S(X_S-\mu_S) +\tilde{Y} + \epsilon_S,
\end{eqnarray*}
which expresses $Y_S$ as the sum of three mutually orthogonal terms.
As $X_S-\mu_S$ is uncorrelated with $\mathcal{O}$, and $\tilde{Y}$ is
uncorrelated with $X_S-\mu_S$, we may calculate the projection of $Y_S$
onto the span of $X_S-\mu_S$ and $\mathcal{O}$ by projecting
$\beta^{\top}_S(X_S-\mu_S)$ onto $X_S-\mu_S$ 
and projecting $\tilde{Y}$ onto $\mathcal{O}$.

We know from (\ref{bfw}) that the projection of
$\beta^{\top}_S(X_S-\mu_S)$ onto $X_S-\mu_S$ 
is $\beta_F^{\top}(X_S-\mu_S)$; in other words, including $V_S$
and the interaction terms does not change $\beta_F$.

For the projection of $\tilde{Y}$ onto $\mathcal{O}$,
let $a_i = \alpha_i + \beta_i^{\top}\mu_i + \bar{\gamma}^{\top}\nu_i$
and $\bar{a} = \sum_ip_ia_i$.
Then,
\begin{eqnarray*}
\tilde{Y}
&=&
\sum_{i=1}^{\bar{S}}\mathbf{1}\{S=i\}(\alpha_i + \beta^{\top}_i\mu_i)
+\sum_{i=1}^{\bar{S}}\mathbf{1}\{S=i\}\gamma^{\top}_iV_S \\
&=&
\sum_{i=1}^{\bar{S}}\mathbf{1}\{S=i\}(\alpha_i + \beta^{\top}_i\mu_i) +
\sum_{i=1}^{\bar{S}}U_i(S)\gamma^{\top}_iV_S + \sum_{i=1}^{\bar{S}}p_i\gamma^{\top}_iV_S \\
&=&
\sum_{i=1}^{\bar{S}}\mathbf{1}\{S=i\}(\alpha_i + \beta^{\top}_i\mu_i+ \bar{\gamma}^{\top}\nu_i) +
\sum_{i=1}^{\bar{S}}U_i(S)\gamma^{\top}_iV_S + \sum_{i=1}^{\bar{S}}p_i\gamma^{\top}_i(V_S -\nu_S)\\
&=&
\bar{a} + \sum_{i=1}^{\bar{S}-1}U_i(S)(a_i-a_{\bar{S}})
+ \sum_{i=1}^{\bar{S}}U_i(S)\gamma^{\top}_iV_S + \bar{\gamma}^{\top}(V_S -\nu_S).
\end{eqnarray*}
Thus, $\tilde{Y}$ is in the span of $\mathcal{O}$, and its coefficient
on $V_S-\nu_S$ is $\bar{\gamma}$.
With all $\var[V_s]$ having full rank, $V_S-\nu_S$ is not spanned
by the other elements of $\mathcal{O}$, so its coefficient
$\bar{\gamma}$ is uniquely determined.
\end{proof}

\begin{proof}[Proposition~\ref{p:non}]
For the first claim, we have
\begin{eqnarray*}
\cov[\hat{Y}_F(X_S)-Y_S,X_S - \E[X_S|S]]
&=&
-\E[(f_1(S)+\epsilon)(X_S - \E[X_S|S])] \\
&=&
-\E[f_1(S)(X_S - \E[X_S|S])] - \E[\epsilon X_S] + \E[\epsilon\E[X_S|S]] \\
&=&
0 + \E[\E[\epsilon|S]\E[X_S|S]] -\E[\E[\epsilon|X_S] X_S]  = 0.
\end{eqnarray*}
For the second claim, we have
\begin{eqnarray*}
\E[(\hat{Y}_{\gamma}(X_S)-Y_S)^2]
&=& \E[(\gamma(X_S) - f_1(S)-\epsilon)^2]\\
&=& \E[(\gamma(X_S) - f_1(S))^2] + \E[\epsilon^2].
\end{eqnarray*}
The last step uses
$$
\E[(\gamma(X_S)-f_1(S))\epsilon]
= \E[(\gamma(X_S)-f_1(S))\E[\epsilon|S]]=0.
$$
It now follows that if $\gamma$ reduces the expected squared forecast error
then $\E[\gamma(X_S)f_1(S)]>0$, which implies (\ref{covnon2}) and (\ref{covnon3}).
\end{proof}

\clearpage

\setcounter{section}{0}
\renewcommand{\thesection}{EC.\arabic{section}}
\renewcommand{\thefigure}{EC.\arabic{figure}}
\renewcommand{\thetable}{EC.\arabic{table}}

\begin{center}
\LARGE

Electronic Companion to \\[18pt]

{\LARGE\bf Should Bank Stress Tests Be Fair?}
\end{center}

\vspace*{0.5in}

\noindent This Electronic Companion covers the following topics.
Section~\ref{s:cross} examines how changes in bank-specific parameters
affect parameters in an aggregated model. We show that in many cases
these parameter externalities favor FEO over a pooled model;
Section~\ref{a:sens} provides supporting analysis.
Section~\ref{a:cxcomb} argues that the only parameter aggregation
rules satisfying some simple 
conditions
are convex combinations,
again supporting the FEO model over the pooled model.
Section~\ref{appendix:data} provides some additional information
on the data used in Section~\ref{sec:experiment}, and
Section~\ref{appendix:robust} provides some robustness checks
supporting Section~\ref{sec:experiment}.
Section~\ref{appendix:ppnr} discusses bank heterogeneity
in revenue models.

\section{Cross-Bank Parameter Externalities}
\label{s:cross}

As a consequence of aggregating bank-specific results
into a single industry model, changes at one bank can 
affect loss forecasts at other banks, and the results
are sometimes counterintuitive.
In this section, we argue that these cross-bank externalities
are generally more reasonable under FEO forecasts than
under the pooled method.

For simplicity, we consider a setting with a single scalar feature $x$.
More generally, we can think of this as a feature that is uncorrelated
with all other features. We adopt the convention that this feature is
nonnegative, and that higher values of $x$ are associated with higher losses.
Thus, for each bank $s$ we assume $\mu_s\ge 0$ and $\beta_s\ge 0$.
In reducing $\mu_s$, a bank improves its portfolio quality;
in reducing $\beta_s$, a bank improves its ability to manage portfolio risk;
and in reducing $\alpha_s$, a bank improves unobserved features to reduces its losses. 
We examine how these improvements --- reductions in $\mu_s$, $\alpha_s$,
and $\beta_s$ --- affect stress test results for bank $s$ and other banks $l$.

We can write the FEO loss forecast (\ref{yfeo}) for bank $l$
evaluated at $X_l=x$ as
\begin{equation}
\hat{Y}_{F,l}(x) = \hat{Y}_F(x) = 
\sum_sp_s(\alpha_s+\beta_s\mu_s) + \beta_F(x-\bar{\mu}),
\label{yfeol}
\end{equation}
with $\beta_F = \sum_ip_i\sigma^2_i\beta_i/\sum_ip_i\sigma^2_i$,
as in (\ref{bfeo1}). The forecast is the same for all banks $l$ because
FEO satisfies equal treatment. It is now easy to see that
\begin{equation}
\frac{\partial\hat{Y}_F(x)}{\partial\mu_s}
= p_s\beta_s - p_s\beta_F \ge 0, \quad\mbox{if and only if $\beta_s\ge\beta_F$};
\label{yexm}
\end{equation}
\begin{equation}
\frac{\partial\hat{Y}_F(x)}{\partial\alpha_s} = p_s\ge 0;
\label{yexa}
\end{equation}
and
\begin{equation}
\frac{\partial\hat{Y}_F(x)}{\partial\beta_s}
= p_s\mu_s + (x-\bar{\mu})p_s\sigma^2_s/\sum_ip_i\sigma^2_i \ge 0,
\quad\mbox{if $x>\bar{\mu}$}.
\label{yexb}
\end{equation}
In (\ref{yexm}) we see that if bank $s$ has above-average (relative to $\beta_F$)
sensitivity to feature $x$, then reducing its average exposure to that feature $\mu_s$
reduces loss forecasts for all banks. Equation (\ref{yexa}) shows a similar
overall benefit if bank $s$ improves on the other dimensions captured by
$\alpha_s$. In (\ref{yexm}) we see that an improvement in risk management at bank $s$, 
corresponding to a reduction in $\beta_s$, reduces loss forecasts
at above-average levels of $x$. If $x$ is part of the stress scenario,
then large values of $x$ are particularly relevant.

The directional effects in (\ref{yexm})--(\ref{yexb}) are fairly simple
and reasonable, considering that cross-bank effects are inevitable
in an industry model. If the industry improves its performance (perhaps
because of improvements at one bank) we generally expect loss forecasts
to decrease. (A decrease in a forecast corresponds to a positive derivative because
we are considering a decrease $\mu_s$, $\alpha_s$, or $\beta_s$.)
Counterparts to (\ref{yexm})--(\ref{yexb}) continue to hold
if we replace $\beta_F$ in (\ref{yfeol}) with any convex combination
of the $\beta_s$, as in the WATE model. However, the pooled method
behaves quite differently.

The pooled forecast $\hat{Y}_P(x)$ can be written in the same
form as (\ref{yfeol}) but with $\beta_F$ replaced by $\beta_{Pool}$
in (\ref{bpool}). We now get
$$
\frac{\partial\hat{Y}_P(x)}{\partial\mu_s}
= p_s(\beta_s -\beta_{Pool}) + (x-\bar{\mu})\frac{\partial\beta_{Pool}}{\partial\mu_s}.
$$
The sign of the last term is not determined by a simple condition, so
the overall directional effect is difficult to predict.
The sign of
$$
\frac{\partial\hat{Y}_P(x)}{\partial\alpha_s}
= p_s + p_s \frac{ (\mu_s - \bar{\mu}) (x - \bar{\mu})}{\sum_s p_s \sigma_s^2 + \var(\mu_S)} \beta_s,
$$
depends on the magnitudes of $\mu_s$ and $x$, relative to $\bar{\mu}$.
For the sensitivity to $\beta_s$, we can write
$$
\frac{\partial\hat{Y}_P(x)}{\partial\beta_s}
= p_s\mu_s + (x-\bar{\mu})\frac{\partial\beta_{Pool}}{\partial{\beta_s}},
\quad
\frac{\partial\beta_{Pool}}{\partial{\beta_s}}
=\frac{p_s(\sigma^2_s +\mu_s(\mu_s-\bar{\mu}))}{\sum_i p_i \sigma_i^2 + \var(\mu_S)}.
$$
Among the most troubling aspects of the pooled model
is that the last term could be negative:
a reduction in $\beta_s$ could produce an increase
in $\beta_{Pool}$.
In particular, $\sigma_s^2+\mu_s(\mu_s-\bar{\mu})$ is negative
for a bank with below-average exposure to feature $x$
(so $\mu_s<\bar{\mu}$) and low variability $\sigma^2_s$ in this exposure.
Under the pooled model, it is therefore possible for
an improvement in risk management at one bank (a reduction in $\beta_s$)
to produce an \emph{increase} in loss forecasts at all banks.

The top panel of Table~\ref{t:sens} shows sufficient conditions
for positive sensitivities of $\hat{Y}_F(x)$ and $\hat{Y}_P(x)$.
The middle and bottom panels show corresponding results for the
expected forecasts $\E[\hat{Y}_l] = \E[\hat{Y}(X_l)]$
and for the bias $\E[\hat{Y}(X_l)-Y_l]$.
Supporting details for the second and third cases are provided in Section~\ref{a:sens}.
We have tried to provide simple sufficient conditions,
and in most cases the conditions are not necessary.
All of the conditions for FEO extend to WATE with $\beta_F$ replaced
by the weighted average coefficient. 

Some counterintuitive and undesirable cases can arise at
empirically plausible parameter values.
For example, in equation (\ref{eqn:sens-a-mean-loss}) 
we derive an expression for 
$\partial\E[\hat{Y}_P(X_l)]/\partial\alpha_s$.
Using estimated parameters for the credit card data in Section~\ref{sec:experiment},
we find that this derivative is negative when 
$l$ is Citigroup and $s$ is JPMorgan Chase. 
In other words, an improvement at JPMorgan Chase would
result in a higher expected loss forecast at Citigroup under the pooled model.

The bias sensitivities in Table~\ref{t:sens} are more complicated
than the other cases because the bias involves the difference
between the predicted and actual loss rates.
A reduction in the predicted loss rate can increase or decrease bias,
depending on whether the initial forecast is too low or too high.

\begin{table}
\centering
\begin{tabular}{l|l|l}
\multicolumn{1}{l}{$\hat{Y}(x)$} & \multicolumn{1}{c}{FEO}  & \multicolumn{1}{c}{Pool}   \\
\hline 
$\mu_s \downarrow$        &  $\downarrow$ iff $\beta_s > \beta_F$  & no simple rule\\
$\alpha_s \downarrow$     & $\downarrow$  & $\downarrow$ if $(\mu_s-\bar{\mu})(x-\bar{\mu})>0$  \\
$\beta_s \downarrow$      & $\downarrow$ if $x>\bar{\mu}$ & $\downarrow$ if $[\sigma_s^2 + \mu_s(\mu_s - \bar{\mu})](x-\bar{\mu}) > 0$ \\ \hline
\multicolumn{3}{c}{}\\[-6pt]
\multicolumn{3}{l}{$\E[\hat{Y}(X_l)]$}\\
\hline
$\mu_s \downarrow$  & $l=s$: $\downarrow$ & no simple rule \\
& $l\ne s$: $\downarrow$ iff $\beta_s > \beta_F$ &  \\
$\alpha_s \downarrow$ & $\downarrow$  & $l=s$: $\downarrow$ \\
& & $l \ne s$: $\downarrow$ if $(\mu_s - \bar{\mu})(\mu_l - \bar{\mu}) > 0$ \\
$\beta_s \downarrow$  & $\downarrow$ if $\mu_s+\mu_l>\bar{\mu}$   & $\downarrow$ if $[\sigma_s^2 + \mu_s(\mu_s - \bar{\mu})](\mu_l-\bar{\mu})  > 0$ \\
&    & or if $\mu_s$ sufficiently large \\ \hline
\multicolumn{3}{c}{}\\[-6pt]
\multicolumn{3}{l}{$\mathsf{bias}(l)$}\\
\hline 
$\mu_s\downarrow$        & $l=s$: $\downarrow$ iff $\beta_s < \beta_F$  & no simple rule \\
& $l \ne s$: $\downarrow$ iff $\beta_s > \beta_F $ & \\
$\alpha_s\downarrow$     & $l=s$: $\uparrow$ & $l=s$: no simple rule \\
& $l \ne s$: $\downarrow$ &  $l\ne s$: $\downarrow$ if $(\mu_s - \bar{\mu})(\mu_l - \bar{\mu}) > 0$ \\ 
$\beta_s\downarrow$      &  $l=s$: $\uparrow$ if $\mu_s < \bar{\mu}$  & no simple rule \\
&  $l\ne s$: $\downarrow$ if $\mu_s + \mu_l>\bar{\mu}$ & $\downarrow$ if $[\sigma_s^2 + \mu_s(\mu_s - \bar{\mu})](\mu_l-\bar{\mu})>0$ \\ \hline
\end{tabular}
\caption{Sensitivity of results for bank $l$ in response to a decrease in parameter
$\mu_s$, $\alpha_s$, or $\beta_s$ for bank $s$. Sensitivities shown are for
predicted loss $\hat{Y}_l(x)$ (top), mean predicted loss $\E[\hat{Y}(X_l)]$ (middle),
and the bias $\E[\hat{Y}(X_l) - Y_l]$.}
\label{t:sens}
\end{table}

\section{Sensitivity Analysis}
\label{a:sens}

This section provides supporting details for Section~\ref{s:cross}, particularly the
conclusions summarized in the middle and bottom panels of Table~\ref{t:sens}.
We begin with an analysis of forecast bias that is of independent interest.

\subsection{Forecast Bias}
\label{a:bias}

If losses at different banks are described by different models, then forecast
bias becomes inevitable when we apply a single model to all banks.
But the distribution of bias across banks may differ under different
choices of the single model.

Let $\hat{Y}_s$ be any of the forecasts for bank $s$ in Table~\ref{t:uni},
and, as in (\ref{ys}), let $Y_s$ denote the actual loss rate for bank $s$.
Both $\hat{Y}_s$ and $Y_s$ are evaluated at $X_s$. Define the
forecast bias for bank $s$ to be
\begin{equation}
\mathsf{bias}(s) = \E[\hat{Y}_s - Y_s].
\label{bias}
\end{equation}
The expectation integrates over the distribution of the error $\epsilon_s$
in (\ref{ys}) and the features $X_s$.

\begin{proposition}
For each forecast in Table~\ref{t:uni}, the bias is as follows.
\begin{itemize}
\item[(i)] Pooled: $\mathsf{bias}(s)=\E[Y_S] -\E[Y_s] + \beta_{Pool}^{\top}(\mu_s-\bar{\mu})$;
\item[(ii)] PTF in (\ref{ptfn}): $\mathsf{bias}(s)=\E[Y_S] -\E[Y_s]$;
\item[(iii)] Conditional expectation: $\mathsf{bias}(s)=\E[\hat{Y}_C(X_s)] -\E[Y_s]$;
\item[(iv)] FEO: $\mathsf{bias}(s)=\E[Y_S] -\E[Y_s] + \beta_F^{\top}(\mu_s-\bar{\mu})$;
\item[(v)] SEO: $\mathsf{bias}(s)=\E[Y_S] -\E[Y_s]$.
\end{itemize}
\label{p:bias}
\end{proposition}

\begin{proof}
For (i), we have, using the definition of $\alpha_{Pool}$ in (\ref{apool}),
\begin{eqnarray*}
\E[\hat{Y}_s - Y_s]
&=& \E[\alpha_{Pool} + \beta^{\top}_{Pool}X_s - Y_s]\\
&=& (\E[Y_S] - \beta^{\top}_{Pool}\bar{\mu}) + \beta^{\top}_{Pool}\mu_s
- \E[Y_s] \\
&=& \E[Y_S]-\E[Y_s] + \beta^{\top}_{Pool}(\mu_s-\bar{\mu}).
\end{eqnarray*}
For the PTF forecast, (\ref{stab}) and (\ref{ptfn}) yield
$$
\E[\hat{Y}_s] = \bar{\alpha}^o = \sum_sp_s\alpha_s^o = \sum_sp_s\E[Y_s]
= \E[Y_S],
$$
and the bias in (ii) follows. 
The expression in (iii) holds by definition.
The argument for (iv) is the same as the argument for (i).
The bias in (v) follows from (iv) because we see from (\ref{yseo})
that the SEO forecast for bank $s$ subtracts $\beta^{\top}_F(\mu_s-\bar{\mu})$ from the FEO forecast.
\end{proof}

In every case of Proposition~\ref{p:bias}, the average bias
$\sum_sp_s\mathsf{bias}(s)$ is zero, but the methods differ
in how they distribute bias across banks. 
We saw previously that the PTF and SEO methods
go the farthest in equalizing differences; we now see that the bias for each
of these methods is the difference $\E[Y_S]-\E[Y_s]$ between
the average loss rate for all banks and the average for an
individual bank.

Using the relationship $\beta_{Pool} = \beta_F + \Lambda\delta$
from (\ref{ovb}), we see that the difference between
the expressions in (i) and (iv) is
$$
\mathsf{bias}_{Pool}(s) - \mathsf{bias}_{FEO}(s)
= \delta^{\top}\Lambda^{\top}(\mu_s - \bar{\mu}).
$$
In light of the discussion in Section~\ref{s:feo},
this difference is the expected disparate impact on bank $s$
of using the pooled model.

\subsection{Improvement in Intercept $\alpha_s$}

By taking the expectation of (\ref{yfeol}), we get
\begin{equation}
\E[\hat{Y}_F(X_l)] = \sum_sps(\alpha_s + \beta_s\mu_s) + \beta_F(\mu_l-\bar{\mu}),
\label{ey}
\end{equation}
and the same holds for the expected pooled forecast with $\beta_F$ replaced by $\beta_{Pool}$.
It follows that, for any banks $s$ and $l$,
\[
\fracp{\E[\hat{Y}_F(X_l)]}{\alpha_s} = p_s > 0.
\]
In other words, all expected forecasts decrease following a reduction in $\alpha_s$.

In contrast, for the pooled model we get
\begin{equation}\label{eqn:sens-a-mean-loss}
    \fracp{\E[\hat{Y}_P(X_l)]}{\alpha_s} = p_s - \frac{{\partial \cov(\alpha_S, \mu_S)}/{\partial\alpha_s}}{\sum_t p_t \sigma_t^2 + \var(\mu_S)}\beta_s (\bar{\mu} - \mu_l) 
= p_s + p_s \frac{ (\mu_s - \bar{\mu}) (\mu_{l} - \bar{\mu})}{\sum_s p_s \sigma_s^2 + \var(\mu_S)} \beta_s.
\end{equation}
Bank $s$ benefits from its reduction of $\alpha_s$, in the sense that the derivative with $l=s$ is positive.
For $l\neq s$, the sign of (\ref{eqn:sens-a-mean-loss}) does not admit a simple description.
In particular, it may be negative when $\mu_s$ and $\mu_l$ are on opposite sides of $\bar{\mu}$,
meaning that one bank's loans are riskier than average and the other bank's loans are less risky
than average.

For the bias under FEO we have
\[
\fracp{\mathsf{bias}_F(l)}{\alpha_s}  = p_s - \1 \{l=s\}
\]
It is then immediate that
\[
\fracp{\mathsf{bias}_F(l)}{\alpha_s} > 0 \text { if } l \ne s \quad \mbox{and}\quad
\fracp{\mathsf{bias}_F(s)}{\alpha_s} < 0.
\]
The direction of change makes sense. 
If the bias for a bank is positive, meaning that the industry model overestimates its losses,
then improvements at other banks will reduce loss forecasts and thus reduce the bias.
The bank's own improvements will increase the bias by reducing the bank's own losses
by more than they reduce the model's forecasts.
The situation is reversed for a bank with a negative bias.

However, for the pooled regression method,
\[
\fracp{\mathsf{bias}_P(l)}{\alpha_s} 
= p_s + p_s \frac{ (\mu_s - \bar{\mu}) (\mu_{l} - \bar{\mu})}{\sum_s p_s \sigma_s^2 + \var(\mu_S)} \beta_s
- \1 \{l=s\},
\]
and the direction of change is unclear.

\subsection{Improvement in Loan Quality}

Now suppose bank $s$ improves the quality of its loan portfolio, resulting in a
smaller $\mu_s$.
This has no effect on $\beta_F$, which makes sense --- changing one bank's
loan quality should not change the sensitivity of losses to loan quality.
However, it is evident from (\ref{bpool}) that $\beta_{Pool}$ does change with $\mu_s$.

Under FEO, the mean the mean predicted loss rate satisfies
\[
\fracp{\E \hat{Y}_F(X_l)}{\mu_s} = p_s \beta_s + \beta_F (\1\{l=s\} - p_s),
\]
which is always positive if $l=s$. 
This means that an improvement in bank $l$'s loan quality (a reduction in $\mu_l$)
reduces bank $l$'s mean predicted losses.
In the pooled model,
\[
\fracp{\E \hat{Y}_{P}(X_l)}{\mu_s} = p_s \beta_s + \beta_{Pool} (\1\{l=s\} - p_s) + (\mu_s - \bar{\mu}) \fracp{\beta_{Pool}}{\mu_s};
\]
this expression could be negative, even with $l=s$, meaning that a bank could be penalized
(through a higher mean predicted loss rate) as a result of improving its loan quality.

The sensitivity of the bias under FEO is given by
\[
\fracp{\mathsf{bias}_F(l)}{\mu_s} = (\1 \{l = s \} - p_s) (\beta_F - \beta_s);
\]
in particular, the bias for bank $l$ moves in opposite directions with respect to changes
in $\mu_l$ and $\mu_s$, $s\not=l$.
Suppose industry model overestimates bank $l$'s losses, in the sense that the bias is positive,
and suppose the industry model overestimates bank $l$'s sensitivity to loan quality,
in the sense that $\beta_F>\beta_l$. Then bank $l$ will benefit (in the sense of reducing the
bias) from improving its loan quality by reducing $\mu_l$.

For the pooled regression,
\[
\fracp{\mathsf{bias}_P(l)}{\mu_s} = 
(\beta_{Pool}-\beta_s) (\1\{l=s\} - p_s) + (\mu_s - \bar{\mu}) \fracp{\beta_{Pool}}{\mu_s}.
\]
The sign of this expression does not admit a simple condition.

\subsection{Improvement in Loan Management}

Now suppose bank $s$ improves its abilities in loan management,
resulting in a reduction in $\beta_s$.
The mean predicted loss rate under FEO satisfies
\[
\fracp{\E[\hat{Y}_{F}(X_l)]}{\beta_s} = p_s (\mu_s + \mu_{l} - \bar{\mu}),
\]
and is positive if $\mu_s+\mu_l>\bar{\mu}$.
In the pooled model
\[
\fracp{\E[\hat{Y}_{P}(X_l)]}{\beta_s} = p_s \mu_s + 
\frac{p_s(\sigma^2_s + \mu_s(\mu_s-\bar{\mu}))}{\sum_ip_i(\sigma^2_i + \mu_i(\mu_i-\bar{\mu}))}
(\mu_{l} - \bar{\mu}),
\]
so $[\sigma^2_s + \mu_s(\mu_s-\bar{\mu})](\mu_l-\bar{\mu})>0$ is a sufficient condition
for the sensitivity to be positive. Regardless of the value of $\mu_l$, the sensitivity
is positive for all sufficiently large $\mu_s$.

For $l\not=s$, the sensitivity of the bias for bank $l$ with respect to $\beta_s$ equals
the sensitivity of the mean predicted loss because the actual expected loss $\E[Y_l]$
is unaffected by $\beta_s$. We therefore focus on the case $l=s$.
Under FEO,
\[
\fracp{\mathsf{bias}_F(s)}{\beta_s} 
= (p_s - 1) \mu_s + p_s (\mu_s - \bar{\mu}),
\]
which is guaranteed to be negative if $\mu_s<\bar{\mu}$.
Under the pooled model, the sign of
\[
\fracp{\mathsf{bias}_P(s)}{\beta_s} 
= (p_s - 1) \mu_s + 
\frac{p_s(\sigma^2_s + \mu_s(\mu_s-\bar{\mu}))}{\sum_ip_i(\sigma^2_i + \mu_i(\mu_i-\bar{\mu}))}
(\mu_s - \bar{\mu})
\]
does not admit a simple characterization.

\section{Convex Combinations of Coefficients}
\label{a:cxcomb}

Equation (\ref{bpool1}) aggregates the individual scalar slopes $\beta_s$
into a single value. We can generalize this perspective 
and ask what properties we would like in an aggregation function,
meaning a function $f:\R^{\bar{S}}\to\R$,
$$
\beta_* = f(\beta_1,\dots,\beta_{\bar{S}}),
$$
that combines bank-specific coefficients $\beta_s$ into an
``industry'' parameter $\beta_*$.

We consider the following properties:
\begin{itemize}
\item[(i)] $f(kb_1,\dots,kb_{\bar{S}}) = kf(b_1,\dots,b_{\bar{S}})$,
for all $k$, $b_1,\dots,b_{\bar{S}}\in\R$;
\item[(ii)] $f(b,\dots,b)=b$, for at least one nonzero $b\in\R$;
\item[(iii)]  $b_s>0$, for all $s$, implies $f(b_1,\dots,b_{\bar{S}}) \ge 0$;
\item[(iv)] $f$ is differentiable at zero.
\end{itemize}

Property (i) is needed for the aggregation to perform sensibly
under a change of units in the measurement of $X_s$:
if we divide each $X_s$ by $k$, each $\beta_s$ increases by a factor of $k$,
and it is natural to require that $\beta_*$ scale accordingly.
Properties (ii) and (iii) are also very modest requirements.
Property (iv) is harder to motivate but not unreasonable.
These properties constrain the aggregation function as follows:

\begin{proposition}
If (i)--(iv) hold, then $f(\beta_1,\dots,\beta_{\bar{S}})$ is a 
convex combination of its arguments.
\label{p:agg}
\end{proposition}

\begin{proof}
Fix $\beta \in \R^{\bar{S}}$. 
Let $g(t) = f(t \beta)$. By condition (iv), $g'(0) = \beta^\top f'(0)$.
Condition (i) and (ii) imply $g(t) = t f(\beta)$,
so $g'(t) = f(\beta)$ for any $t$. Thus, 
$f(\beta) = g'(0) = \beta^\top f'(0)=\sum_{i=1}^{\bar{S}} f_i'(0)\beta_i$.
Condition (ii) now implies $\sum_{i=1}^{\bar{S}} f_i'(0) = 1$,
and condition (iii) implies $f_i'(0) \ge 0$, for all $i$.
Thus, $f(\beta) = \sum_{i=1}^{\bar{S}} f_i'(0)\beta_i$ is a convex
combination of the components of $\beta$.
\end{proof}

The scalar FEO coefficient in (\ref{bfeo1}) is a convex combination 
of the bank-specific coefficients $\beta_s$,
but the pooled coefficient (\ref{bpool1}) is generally not.
This property of the FEO model extends to the multivariate case 
under additional conditions. If all the bank-specific covariance matrices 
$\Sigma_s$, $s=1,\dots,\bar{S}$, coincide,
then in (\ref{bfeo}) we get $\beta_F = \E[\beta_S] = \sum_sp_s\beta_s$.
If all $\Sigma_s$ are diagonal (but not necessarily identical), 
then the representation of the scalar
FEO coefficient in (\ref{bfeo1}) applies to each coordinate of $\beta_F$.
If all $\Sigma_s$ have the same eigenvectors, 
then we can transform the original features $X_s$
into uncorrelated features using principal components. 
Using these transformed features, each coordinate of $\beta_F$ 
is a convex combination of bank-specific coefficients.

\section{Additional Information on Empirical Analysis}
\label{appendix:data}

Table~\ref{t:names} lists the bank holding companies included in 
our empirical analysis and the symbols we use to refer to them.
The companies are listed in order of size by total assets.

\begin{table}[H]
\centering
\begin{tabular}{ll}
\toprule
\textbf{Ticker}	&	\textbf{Bank Name}	\\ 
\midrule
    JPM &                     JPMORGAN CHASE \& CO. \\
    BAC &              BANK OF AMERICA CORPORATION \\
      C &                           CITIGROUP INC. \\
    WFC &                    WELLS FARGO \& COMPANY \\
     GS &           GOLDMAN SACHS GROUP, INC., THE \\
     MS &                           MORGAN STANLEY \\
   SCHW &          CHARLES SCHWAB CORPORATION, THE \\
    USB &                             U.S. BANCORP \\
    PNC &  PNC FINANCIAL SERVICES GROUP, INC., THE \\
    TFC &             TRUIST FINANCIAL CORPORATION \\
     TD &                 TD GROUP US HOLDINGS LLC \\
     BK & BANK OF NEW YORK MELLON CORPORATION, THE \\
    COF &        CAPITAL ONE FINANCIAL CORPORATION \\
    STT &                 STATE STREET CORPORATION \\
   HSBC &         HSBC NORTH AMERICA HOLDINGS INC. \\
   SIVB &                      SVB FINANCIAL GROUP \\
   FITB &                      FIFTH THIRD BANCORP \\
   USAA &   UNITED SERVICES AUTOMOBILE ASSOCIATION \\
    BMO &                      BMO FINANCIAL CORP. \\
    CFG &           CITIZENS FINANCIAL GROUP, INC. \\
    AXP &                 AMERICAN EXPRESS COMPANY \\
    KEY &                                  KEYCORP \\
   NTRS &               NORTHERN TRUST CORPORATION \\
   ALLY &                      ALLY FINANCIAL INC. \\
    AMP &               AMERIPRISE FINANCIAL, INC. \\
     RY &                RBC US GROUP HOLDINGS LLC \\
   HBAN &       HUNTINGTON BANCSHARES INCORPORATED \\
     RF &            REGIONS FINANCIAL CORPORATION \\
   MUFG &       MUFG AMERICAS HOLDINGS CORPORATION \\
    BCS &                          BARCLAYS US LLC \\
    SAN &             SANTANDER HOLDINGS USA, INC. \\
    MTB &                     M\&T BANK CORPORATION \\
  BNPQY &                    BNP PARIBAS USA, INC. \\
     DB &                       DB USA CORPORATION \\
    DFS &              DISCOVER FINANCIAL SERVICES \\
\bottomrule
\end{tabular}
\caption{Symbols and names of included bank holding companies.}
\label{t:names}
\end{table}

We construct the loss rates, past due rates, and allowance rates using the entries in FR Y-9C forms outlined in Table~\ref{tbl:fry9c}.
\begin{table}[H]
\centering
\begin{tabular}{l|c|cc}
\toprule
\textbf{Variables} & \textbf{Loan Types} & \textbf{2007Q1 – Present} & \textbf{2003Q1 – 2006Q4} \\
\hline
\multirow{5}{*}{Loan Amount} & CC & BHCKB538 & BHCKB538 \\
 & FL & BHDM5367 & BHDM5367 \\
 & CRE & \begin{tabular}[c]{@{}l@{}}Owned: BHCKF160\\ Other: BHCKF161\end{tabular} & BHDM1480 \\
 & CI & BHCK1763 & BHCK1763 \\
 & Total & BHCK2122 & BHCK2122 \\
 \hline
\multirow{4}{*}{Charge-Offs} & CC & BHCKB514 & BHCKB514 \\
 & FL & BHCKC234 & BHCKC234 \\
 & CRE & \begin{tabular}[c]{@{}l@{}}Owned: BHCKC895\\ Other: BHCKC897\end{tabular} & BHCK3590 \\
 & CI & BHCK4645 & BHCK4645 \\
 \hline
\multirow{4}{*}{Recoveries} & CC & BHCKB515 & BHCKB515 \\
 & FL & BHCKC217 & BHCKC217 \\
 & CRE & \begin{tabular}[c]{@{}l@{}}Owned: BHCKC896\\ Other: BHCKC898\end{tabular} & BHCK3591 \\
 & CI & BHCK4617 & BHCK4617 \\
 \hline
\multirow{4}{*}{Past Due: 30-89 days and accruing} & CC & BHCKB575 & BHCKB575 \\
 & FL & BHCKC236 & BHCKC236 \\
 & CRE & \begin{tabular}[c]{@{}l@{}}Owned: BHCKF178\\ Other: BHCKF179\end{tabular} & BHCK3502 \\
 & CI & BHCK1606 & BHCK1606 \\
 \hline
\multirow{4}{*}{Past Due: 90 days and accruing} & CC & BHCKB576 & BHCKB576 \\
 & FL & BHCKC237 & BHCKC237 \\
 & CRE & \begin{tabular}[c]{@{}l@{}}Owned: BHCKF180\\ Other: BHCKF181\end{tabular} & BHCK3503 \\
 & CI & BHCK1607 & BHCK1607 \\
 \hline
\multirow{4}{*}{Past Due: non-accrual} & CC & BHCKB577 & BHCKB577 \\
 & FL & BHCKC229 & BHCKC229 \\
 & CRE & \begin{tabular}[c]{@{}l@{}}Owned: BHCKF182\\ Other: BHCKF183\end{tabular} & BHCK3504 \\
 & CI & BHCK1608 & BHCK1608 \\
\bottomrule
\end{tabular}
\caption{Loan variables and FR Y-9C form correspondence.}\label{tbl:fry9c}
\end{table}

\section{Robustness Checks}
\label{appendix:robust}

\subsection{Pre-COVID Data and Allowance Rate}
\label{appendix:robust-precovid}
We repeat the analysis of Section~\ref{sec:experiment}, 
limiting the data to 2001--2019.
This serves two purposes. It addresses the possibility that our results are
driven by a few extreme values during the COVID period 2020--2021. 
It also accounts for a change in how banks measure allowances 
(the Current Expected Credit Losses methodology) that began to take effect at the end of 2019. 
We also consider adding allowance rate as another proxy for portfolio risks. 
Tables~\ref{tbl:1y_het_precovid} and ~\ref{tbl:1y_precovid} report the results for 
heterogeneity tests and differences of the parameter estimates under this setting. 
The evidence for heterogeneity and its impact is generally 
at least as strong using the pre-COVID data as using data through 2021.

\begin{landscape}
\begin{table}[]
    \centering
   \begin{tabular}{l|rrrrrrrrrrrrrrrr}
\toprule
Covariance & \multicolumn{4}{|c|}{$\alpha$} & 
\multicolumn{4}{|c|}{$\beta_{PDR}$} & \multicolumn{4}{|c}{$\beta_{AR}$} & \multicolumn{4}{|c}{$\gamma$}\\
Estimation &    \multicolumn{1}{|c}{CC} &   FL &  CRE &     CI &     \multicolumn{1}{|c}{CC} &   FL &  CRE &     CI &     \multicolumn{1}{|c}{CC} &   FL &  CRE &     CI &     \multicolumn{1}{|c}{CC} &     FL &    CRE &     CI \\
\midrule
bank clustered &  0.00 &  0.00 &  0.00 &  0.00 &  0.00 &  0.00 &  0.00 &  0.00 &       &       &       &       &       &       &       &       \\
time clustered &  0.00 &  0.00 &  0.00 &  0.00 &  0.00 &  0.00 &  0.00 &  0.00 &       &       &       &       &       &       &       &       \\ \hline 
bank clustered &  0.00 &  0.00 &  0.00 &  0.00 &  0.00 &  0.00 &  0.00 &  0.00 &  0.00 &  0.00 &  0.00 &  0.00 &       &       &       &       \\
time clustered &  0.00 &  0.00 &  0.00 &  0.00 &  0.00 &  0.00 &  0.00 &  0.00 &  0.00 &  0.00 &  0.00 &  0.00 &       &       &       &       \\ \hline 
bank clustered &  0.00 &  0.00 &  0.00 &  0.00 &  0.00 &  0.00 &  0.00 &  0.00 &       &       &       &       &  0.00 &  0.00 &  0.00 &  0.00 \\
time clustered &  0.00 &  0.00 &  0.00 &  0.00 &  0.00 &  0.00 &  0.00 &  0.00 &       &       &       &       &  0.00 &  0.00 &  0.00 &  0.00 \\ \hline 
bank clustered &  0.00 &  0.00 &  0.00 &  0.00 &  0.00 &  0.00 &  0.00 &  0.00 &  0.00 &  0.00 &  0.00 &  0.00 &  0.00 &  0.00 &  0.00 &  0.00 \\
time clustered &  0.00 &  0.00 &  0.00 &  0.00 &  0.00 &  0.00 &  0.00 &  0.00 &  0.00 &  0.00 &  0.00 &  0.00 &  0.00 &  0.00 &  0.00 &  0.00 \\
\bottomrule
\end{tabular}
    \caption{Heterogeneity tests using pre-COVID data with allowance rate as an additional proxy for banks' portfolio risks.}
    \label{tbl:1y_het_precovid}
\end{table}

\begin{table}[]
\centering
\begin{tabular}{l|rrrlrrrlrrrl}
\toprule
Loan & \multicolumn{4}{|c|}{\textit{Past Due Rate}} & 
\multicolumn{4}{|c|}{\textit{Allowance Rate}} & \multicolumn{4}{|c}{\textit{Macro PC}} \\
Type &  $\beta_{Pool}$ &  $\beta_F$ & diff &   \multicolumn{1}{c|}{$p$-value} & 
$\beta_{Pool}$ & $\beta_{F}$ & diff &              \multicolumn{1}{c|}{$p$-value} & 
$\gamma_{Pool}$ & $\gamma_{F}$ & diff &             $p$-value \\
\midrule
 CC & 0.781 & 0.838 & -0.057 & 0.006$^{***}$ &        &        &        &               &       &       &        &       \\
 FL & 0.045 & 0.061 & -0.016 & 0.002$^{***}$ &        &        &        &       &       &       &        &       \\
CRE & 0.135 & 0.133 &  0.002 & 0.449         &        &        &        &       &       &       &        &       \\
 CI & 0.207 & 0.217 & -0.011 & 0.069$^{*}$   &        &        &        &       &       &       &        &       \\ \hline 
 CC & 0.794 & 0.852 & -0.058 & 0.109         & -0.029 & -0.032 &  0.003 & 0.938 &       &       &        &       \\
 FL & 0.010 & 0.021 & -0.011 & 0.306         &  0.335 &  0.307 &  0.028 & 0.424 &       &       &        &       \\
CRE & 0.095 & 0.068 &  0.026 & 0.052$^{*}$   &  0.127 &  0.191 & -0.064 & 0.090$^{*}$   &       &       &        &       \\ 
 CI & 0.178 & 0.196 & -0.018 & 0.000$^{***}$ &  0.064 &  0.054 &  0.010 & 0.514         &       &       &        &       \\\hline 
 CC & 0.595 & 0.633 & -0.039 & 0.001$^{***}$ &        &        &        &               & 0.417 & 0.397 &  0.020 & 0.000$^{***}$ \\
 FL & 0.037 & 0.052 & -0.015 & 0.001$^{***}$ &        &        &        &               & 0.110 & 0.106 &  0.003 & 0.329 \\
CRE & 0.129 & 0.126 &  0.002 & 0.477         &        &        &        &               & 0.057 & 0.057 &  0.000 & 0.601 \\
 CI & 0.155 & 0.155 & -0.001 & 0.886         &        &        &        &               & 0.190 & 0.190 &  0.000 & 0.981 \\\hline 
 CC & 0.577 & 0.618 & -0.041 & 0.046$^{**}$  &  0.037 &  0.032 &  0.005 & 0.846         & 0.420 & 0.399 &  0.021 & 0.000$^{***}$ \\
 FL & 0.006 & 0.018 & -0.012 & 0.146         &  0.315 &  0.268 &  0.047 & 0.041$^{**}$  & 0.095 & 0.094 &  0.001 & 0.768 \\
CRE & 0.097 & 0.078 &  0.019 & 0.182         &  0.102 &  0.147 & -0.045 & 0.265         & 0.052 & 0.050 &  0.002 & 0.267 \\
 CI & 0.140 & 0.153 & -0.012 & 0.035$^{**}$  &  0.032 &  0.007 &  0.025 & 0.105         & 0.189 & 0.189 & -0.001 & 0.583 \\
\bottomrule
\multicolumn{13}{r}{$^{*}$p$<$0.1; $^{**}$p$<$0.05; $^{***}$p$<$0.01} \\
\end{tabular}
\caption{Comparison of pooled and FEO coefficients using pre-COVID data with allowance rate as an additional proxy for banks' portfolio risks.} 
\label{tbl:1y_precovid}
\end{table}
\end{landscape}

\section{Revenue Models}\label{appendix:ppnr}

The Federal Reserve's stress testing framework includes models of revenues
as well as models of losses. We have focused on loan portfolio loss models
because they fit most clearly within the Fed's policy of equal treatment
and its preference for industry models. In this section, we show that
the heterogeneity documented for loss models in Section~\ref{sec:experiment}
extends to revenue models, referred to in the Fed's framework as
models of pre-provision net revenue or PPNR.

The Fed uses a suite of PPNR models to forecast difference sources of revenue.
These models differ from the loss models in at least two important respects:
they are typically autoregressive (AR) models, and, 
unlike the portfolio loss models, they do not rule out
bank fixed effects; see \cite{frb}. 
This feature points to the presence of unmodeled 
bank heterogeneity in the revenue forecasts.
Our goal in this section is to check for heterogeneity in a simple PPNR model 
and to compare coefficient estimates in the pooled and fixed-effect models.

We consider the modeling of trading revenue, which is one of the PPNR components
in the Fed's framework. We compare AR models with fixed-effects,
\begin{equation}\label{eqn:ppnr_fe}
    Y_{s,t}^{FE} = \alpha_s + \rho_{FE} Y_{s, t-1} + \beta_{FE} X_{s,t} + \gamma_{FE} \textit{VIX}_{t} + \epsilon^{FE}_{s,t}
\end{equation}
or pooled without fixed effects,
\begin{equation}\label{eqn:ppnr_p}
Y_{s,t}^P = \alpha_P + \rho_P Y_{s, t-1} + \beta_P X_{s,t} + \gamma_P \textit{VIX}_t + \epsilon_{s,t}^P.
\end{equation}
In both models, $Y$ is trading revenue normalized by total trading assets;
this choice of normalization is consistent with \cite{frb}. 
For the AR term, we use either a one-quarter lag $Y_{s,t-1}$ or
a four-quarter average lag, in which case we use
$1/4 \sum_{j=1}^4 Y_{s,t-j}$ in place of $Y_{s, t-1}$ in 
(\ref{eqn:ppnr_fe}) and (\ref{eqn:ppnr_p}) to capture 
average performance over the past year.

For $X_{s,t}$ we use the size of bank $s$ in quarter $t$,
as measured by the log of total assets.
For the macro variable, we use the \textit{VIX}$_t$, the market volatility index
taken from the Federal Reserve's stress testing historical dataset. 
Market volatility is expected to have a direct impact on trading revenue,
and indeed we observe a more significant effect of  \textit{VIX}$_t$ than
\textit{MacroPC}$_t$ (from Section~\ref{sec:experiment}) in this setting. 

Models (\ref{eqn:ppnr_fe}) and (\ref{eqn:ppnr_p}) differ in their intercept terms:
(\ref{eqn:ppnr_fe}) captures banks' fixed effects, but (\ref{eqn:ppnr_p}) 
requires the same intercept across all banks. 
The fixed-effect coefficient estimates $\rho_{FE}, \beta_{FE}$ and $\gamma_{FE}$ 
are identical to those of FEO; 
the methods differ in their forecasts:
FEO uses the average fixed effect, rather than bank-specific fixed effects in its forecasts.

As in Section~\ref{sec:experiment}, we use Y-9C financial
reporting data for the top 35 banks by total asset size (as of year-end 2021),
and we include only banks with at least 18 years of data to ensure our panel is 
mostly balanced. 
We fit the models using weighted least squares, weighting each observation
by quarter stress and bank asset balance. As in (\ref{eqn:weight}),
we choose the weights so that, for the same asset level, 
the quarters with the highest market volatility get twice the weight as
the quarters with the lowest market volatility.
As in the AR models in \cite{class}, we cluster standard errors by time.

Table~\ref{tbl:ppnr-trading-rev} reports the results. 
The first two columns correspond to the one-quarter AR setting,
and the last two columns are the one-year-average AR setting.
The numbers in parentheses are standard errors.

As expected, both the lagged response and the
\textit{VIX} term are statistically significant. 
The volatility term is more significant in the one-year-average AR setting,
presumably because the market environment changes less over one quarter, 
and its effect is partly captured by the lagged response.

The pooled and fixed-effect methods result in different estimates of $\rho$,
and the differences in estimates in all settings are more than 20\%.
We tested the hypotheses that all banks share the same (i) intercept,
(ii) AR coefficient term, (iii) \textit{VIX} coefficient, 
or (iv) coefficients of total asset size,
following the approach used in Section~\ref{s:heterogeneity};
all tests strongly reject that banks have identical model coefficients,
with $p$-values less than 0.01,
extending what we found for loss models. 
We have also examined the forecasts of trading interest income and
trading interest expense (two other components of PPNR),
and observed similar patterns across all three components.

\begin{table}[!htbp] 
\centering
\begin{tabular}{@{\extracolsep{5pt}}lcccc}
\hline \\[-1.8ex]
& \multicolumn{4}{c}{\textit{Normalized Trading Revenue}} \
\cr \cline{4-5}
\hline \\[-1.8ex]
 AR & 0.466$^{***}$ & 0.588$^{***}$ & 0.406$^{***}$ & 0.596$^{***}$ \\
  & (0.134) & (0.111) & (0.078) & (0.094) \\
 \textit{VIX} & -0.024$^{**}$ & -0.022$^{*}$ & -0.027$^{***}$ & -0.022$^{***}$ \\
  & (0.011) & (0.012) & (0.008) & (0.008) \\
 Log Assets & -0.003$^{}$ & -0.007$^{***}$ & 0.000$^{}$ & -0.006$^{***}$ \\
  & (0.004) & (0.001) & (0.003) & (0.002) \\
  \hline
 Bank FE & included & -- & included & -- \\
\hline \\[-1.8ex]
 & \multicolumn{4}{r}{$^{*}$p$<$0.1; $^{**}$p$<$0.05; $^{***}$p$<$0.01} \\
\end{tabular}
\caption{Coefficient estimates for the AR models. The first two columns use a one-quarter lag, and the last two use a one-year average lag. Columns 1 and 3 correspond to AR models with bank fixed effects, and columns 2 and 4 are pooled AR models.}
\label{tbl:ppnr-trading-rev}
\end{table}

\end{document}